\documentclass{article} 
\usepackage{iclr2027_conference,times}

\usepackage{amsmath,amsfonts,bm}

\def\eqref#1{equation~\ref{#1}}
\def\Eqref#1{Equation~\ref{#1}}

\def\1{\bm{1}}

\DeclareMathAlphabet{\mathsfit}{\encodingdefault}{\sfdefault}{m}{sl}
\SetMathAlphabet{\mathsfit}{bold}{\encodingdefault}{\sfdefault}{bx}{n}

\usepackage[utf8]{inputenc} 
\usepackage[T1]{fontenc}    
\usepackage[pagebackref,breaklinks,colorlinks]{hyperref}
\usepackage{url}            
\usepackage{booktabs}       
\usepackage{amsfonts}       
\usepackage{nicefrac}       
\usepackage{microtype}      
\usepackage[table]{xcolor}         
\usepackage{amsmath}
\usepackage{amsthm}
\theoremstyle{definition} 
\newtheorem{theorem}{Theorem}
\usepackage{enumitem}
\usepackage{multirow}
\usepackage{caption}
\usepackage{float}
\usepackage{subcaption}
\usepackage{arydshln}
\usepackage[pdftex]{graphicx}

\title{MLLMs Hallucinate when Information Distribution Drifts in Synergy Heads}

\author{\textbf{Meng'en Qin}$^{1}$, \textbf{Junye Chen}$^{1}$, \textbf{Jucheng Liu}$^{1}$, \textbf{Yinchen Liu}$^{2}$, \textbf{Youlu Xing}$^{1}$,\\ \textbf{Song Wang}$^{1}$, \textbf{Ruize Han}\thanks{Corresponding author.}$^{*,1}$ \\
$^{1}$Shenzhen University of Advanced Technology, Shenzhen, China\\
$^{2}$University of Electronic Science and Technology of China, Chengdu, China \\
\texttt{mengenching@gmail.com}, \texttt{\{wangsong,hanruize\}@suat-sz.edu.cn} \\
}
\iclrfinalcopy 
\begin{document}

\maketitle

\begin{abstract}
Multimodal Large Language Models (MLLMs) often struggle with hallucinations, thus hindering their reliable practical applications. Existing attention-based mitigation methods mainly rely on indirect signals (e.g., attention weights) that fail to accurately reflect the actual information shift underlying hallucination generation. In this paper, we propose HEAL, \textbf{H}ead-l\textbf{E}vel information disent\textbf{A}nglement and ca\textbf{L}ibration for identifying and mitigating hallucinations. HEAL first employs causal noise intervention on multi-head outputs to filter out causally redundant heads. Subsequently, it disentangles information distribution within the remaining heads via the counterfactual Difference-in-Differences, categorizing heads into four types. Through analysis, we observe: \textit{hallucinations happen when information distribution drifts away from a healthy equilibrium in synergy heads}, not strongly correlated with the quantity or strength of modality-specific heads. Motivated by this insight, HEAL injects dynamic information calibration factors into the value vectors of synergy heads and actively regulates visual-language dependencies, steering the output distribution towards factual evidence. Extensive experiments demonstrate that HEAL effectively reduces hallucinations across multiple MLLMs, offering a simple and interpretable pathway to enhance model trustworthiness. 
\end{abstract}

\section{Introduction}
Though Multimodal Large Language Models (MLLMs) have demonstrated remarkable capabilities across a wide range of complex tasks \citep{qwen3vl, llava_1.5, llavanext}, they are frequently plagued by hallucinations \citep{why_hallu, mm_halu_det_survey, mm_hallu_mitigation_survey}, which means generating plausible-sounding responses that are factually inconsistent with the visual context. This phenomenon remains a significant bottleneck for the reliable applications of MLLMs in high-precision fields like medical imaging.

To mitigate hallucinations, extensive efforts have been made from both macro- and micro-level perspectives \citep{mm_hallu_mitigation_survey}. Macro-level strategies, such as retraining or fine-tuning \citep{reflective_fine_tuning, hallucidoctor, perturbollava}, architectural scaling \citep{vcoder, internvl}, and reinforcement learning \citep{hallu_dpo}, typically treat the model as a black box and attempt to correct hallucinations from the outside. In contrast, micro-level methods aim to alleviate hallucinations through decoding optimization \citep{leng2024vcd, convis_decoding, zhu2024ibd} or attention-based enhancement \citep{vhr, dual_attention, inference_intervention}. However, decoding optimization intervenes only at the output level, offering limited interpretability and insight into the internal mechanisms. Likewise, most attention-based approaches depend on indirect signals, such as attention weights or FFN activations, which may be unable to capture the actual causal information distribution and confounding synergy effects involved in hallucination generation. Moreover, uni-modal attention enhancement often leads to a “see-saw” dilemma \citep{dual_attention}: strengthening visual attention may cause truncated or less creative responses, whereas over-reliance on language can exacerbate visual hallucinations.
\begin{figure}[t]
    \centering
    \includegraphics[width=0.90\linewidth]{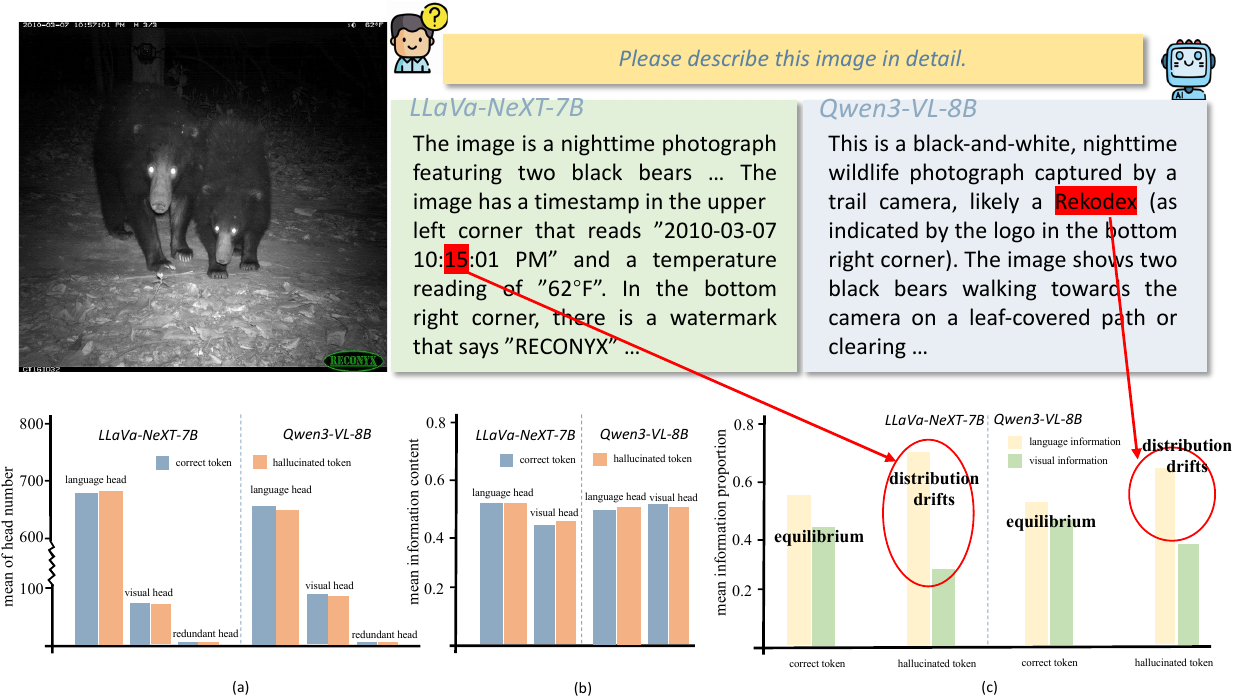}
    \caption{(a) Mean number of heads of different types when generating hallucinated versus correct tokens; (b) Mean information content of visual and language heads in hallucinated and correct token generation; (c) Mean information proportion within synergy heads for hallucinated and correct tokens.}
    \label{fig:hallu_synergy}
\end{figure}

Motivated by these limitations, we propose HEAL, head-level information disentanglement and calibration, to identify and mitigate hallucinations during autoregressive generation. It disentangles the information distribution within attention heads, categorizing them into redundant, visual, language and synergy heads. By leveraging causal noise intervention and counterfactual Difference-in-Differences, HEAL provides a fine-grained, head-level interpretable lens into the internal mechanics of MLLMs. As discussed in Figure \ref{fig:hallu_synergy} and Section \ref{bidirectional_causal}, through extensive analysis on models like Qwen3-VL \citep{qwen3vl} and LLaVA-NeXT \citep{llavanext}, we observe critical insights into the pathology of hallucinations:

\textbf{MLLMs hallucinate when information distribution drifts away from a healthy equilibrium in synergy heads, not strongly correlated with the quantity or strength of modality-specific heads.}

Additionally, from Figure \ref{fig:hallu_dist}(a) and (b), we further see that head roles are highly dynamic throughout the generation. As the model generates different types of tokens, attention heads undergo a task-driven phase transition. When producing language-centric tokens, synergy heads tend to revert toward language heads; when generating visually grounded tokens, some language heads shift toward synergy heads, and a small subset of synergy heads may further move toward visual heads. This dynamic behavior highlights the adaptive nature of Transformer attention \citep{attention}. However, when focusing specifically on visually grounded token generation, we find that hallucinations often begin with an internal drift of information distribution within synergy heads.

Based on these insights, we propose a dynamic information calibration strategy during inference. We introduce an equilibrium factor ($\alpha$), akin to the temperature hyperparameter, to regulate the model’s dependence on visual versus language information. By dynamically monitoring and calibrating the information distribution within synergy heads, HEAL effectively steers the output back toward factual evidence without sacrificing too much linguistic coherence. Extensive experiments across multiple MLLM benchmarks demonstrate that HEAL consistently reduces hallucinations and provides a simple, interpretable pathway toward more trustworthy multimodal systems.

Our contributions are summarized as follows:
\begin{itemize}[leftmargin=*]
    \item We propose HEAL, utilizing causal noise intervention and counterfactual Difference-in-Differences to categorize attention heads into four functional types, providing a new interpretable perspective on MLLM internals.
    \item We identify visual-language information disequilibrium within synergy heads as an intervenable factor that can influence hallucination behavior, rather than attributing hallucinations solely to macro-level modality heads, and characterize the ''task-driven role transition'' of attention heads.
    \item We introduce a theoretically grounded dynamic calibration strategy that uses the equilibrium factor $\alpha$ to regulate the visual-language information distribution, predictably steering the representation geometry toward the equilibrium direction and improving the factuality of MLLM generation.
\end{itemize}
\begin{figure}[ht]
    \centering
    \includegraphics[width=0.9\linewidth]{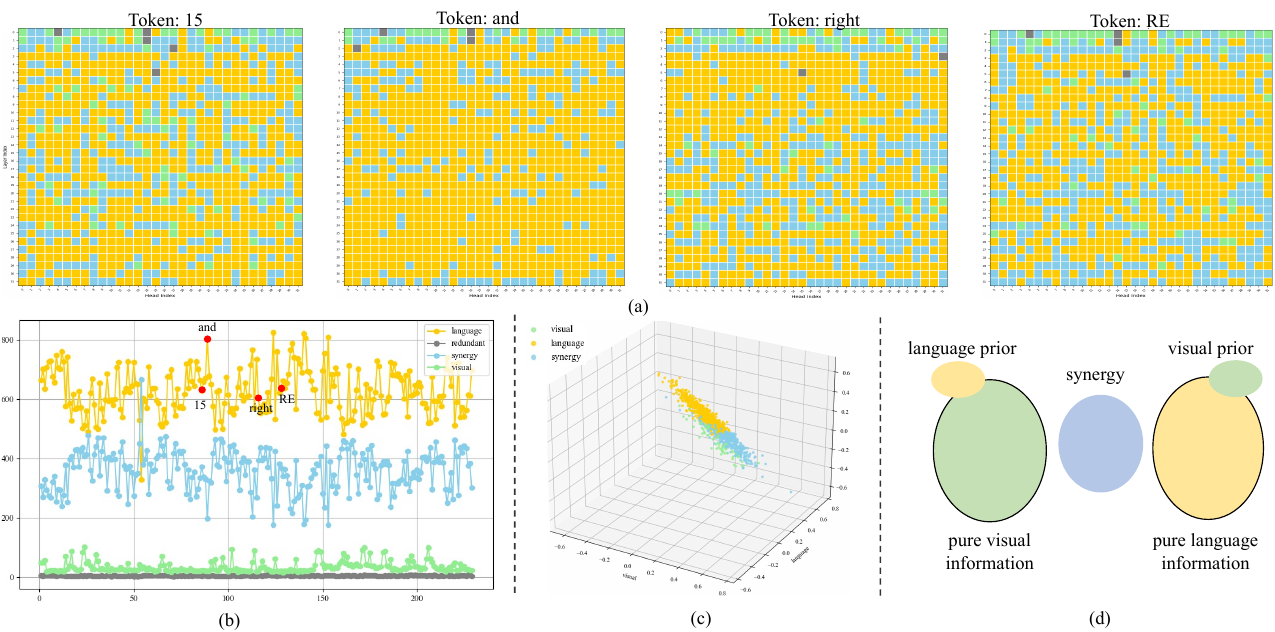}
    \caption{(a) Head distribution in the LLaVA-NeXT-7B model when generating different tokens, where green, blue, yellow and gray denote visual, synergy, language and redundant heads, respectively; (b) The evolution of the number of heads in the autoregressive generation. (c) 3D scatter plot of attention heads of different types, where the axes correspond to visual, language and synergy contributions. (d) Illustration of the internal information structure within a single attention head. The visualization in (a) and (b) is based on the same VQA example as in Figure \ref{fig:hallu_synergy}. \textit{From the figure, we observe that the number of visual heads remains nearly constant throughout the generation.}}
    \label{fig:hallu_dist}
\end{figure}
\section{Related Works}
\textbf{Hallucination Mitigation in MLLMs.}
Recent studies mitigate hallucinations in MLLMs mainly from four perspectives \citep{mm_hallu_mitigation_survey}. Data-centric methods reduce spurious correlations by constructing negative, counterfactual, or cleaner supervision data~\citep{lrv2023,lure2023}. Training- and model-level methods strengthen visual grounding through improved alignment objectives, auxiliary supervision, preference learning, and stronger multimodal architectures~\citep{sun2024farlhf, hallu_dpo, reflective_fine_tuning,hallucidoctor, perturbollava}. Inference-time methods intervene in decoding by contrastive or guided strategies to suppress language priors and encourage visual evidence usage~\citep{leng2024vcd,huang2024opera,convis_decoding}. In addition, detection-based and post-hoc correction methods serve as a practical complement by locating hallucinated content and revising unreliable outputs~\citep{lure2023}.

\textbf{Attention-based Mitigation Methods.}
Most attention-based methods mitigate hallucinations by explicitly increasing visual attention during decoding. The main intuition is that MLLMs tend to rely less on image prompts as generation proceeds, causing language priors to dominate and leading to visually ungrounded outputs \citep{mm_hallu_mitigation_survey}. Representative approaches differ in how they select and strengthen visual attention. AGLA \citep{agla} improves prompt-relevant local attention by combining global and local attention to emphasize regions most relevant to the query. PAI \citep{pai2024}, EAH \citep{eah2024}, and VHR \citep{vhr} further intervene on attention heads to make decoding more image-centric and alleviate visual attention sinks \citep{see_what, attention_sink}. Owl \citep{dual_attention} introduces a dual-path contrastive decoding strategy in which one path reinforces visually grounded attention while the other amplifies hallucinated ones. CausalMM \citep{attention_causality} uses structural causal modeling to treat modality priors as a confounder between attention and output in MLLMs. FarSight \citep{far_see} proposes a versatile plug-and-play decoding strategy that reduces attention interference from outlier tokens merely by optimizing the causal mask.
As discussed above, these approaches tend to rely on attention weights as an indirect proxy for token contribution and cannot reflect the actual information structure in attention heads, while uni-modal attention enhancement inherently struggles to balance vision and language.

\section{Method}
In this section, we present the proposed HEAL and introduce how it (1) identifies and categorizes attention heads, and (2) dynamically calibrates information drift within synergy heads.
\subsection{Causal Noise Intervention on Multi-head Outputs}
\label{causal_intervention}
In MLLMs, the image is first encoded by a visual encoder into visual embeddings, which are then projected into the language space via a projector. The resulting visual tokens are concatenated with text tokens and fed into the language backbone for autoregressive generation. Generally, the language backbone consists of multiple transformer decoder layers and each layer $l$ performs a multi-head attention operation among tokens. Under the KV-cache setting, the attention head $i$ at generation step $t$ is formulated as:
\begin{equation}
\label{equ:attention}
\mathbf{o}^{(t,l,i)} = \mathrm{Softmax}\left(
\frac{\mathbf{q}^{(t,l,i)} (\mathbf{K}^{(t,l,i)})^\top}{\sqrt{d}}
\right)
\begin{bmatrix} \mathbf{V}^{(t,l,i)}_{vis} \\ \mathbf{V}^{(t,l,i)}_{lang} \end{bmatrix},
\end{equation}
where $\mathbf{q}^{(t,l,i)}$ is the query at the current step, $d$ is the dimension of the query, and $\mathbf{K}^{(t,l,i)}$, $\mathbf{V}^{(t,l,i)}_{vis}$, $\mathbf{V}^{(t,l,i)}_{lang}$ denote the cached key, visual value, language value matrices. 
The outputs of all heads are concatenated and linearly projected:
\begin{equation}
\label{equ:multi_head_output}
\mathbf{y}^{(t,l)} = \mathbf{x}^{(t,l)} + 
\mathbf{W}_O^{(l)} \left[
\mathbf{o}^{(t,l,1)}; \cdots; \mathbf{o}^{(t,l,H)}
\right],
\end{equation}
where $\mathbf{x}^{(t,l)}$ is the residual input and $\mathbf{W}_O^{(l)}$ is the output projection matrix.

Our goal is to identify insignificant attention heads that contribute less to the outputs. Instead of relying on the corresponding projection weights in $\mathbf{W}_O^{(l)}$, we directly intervene on the head outputs and measure the resulting differences in the layer representation. This is necessary because a head with a large activation may still contribute substantially even if its corresponding projection weights are small, while a head with large weights but near-zero activations may have a limited effect. In addition, a head's contribution cannot be faithfully captured by its own weights alone, since the final output may depend on the cooperative interaction among multiple heads.

To estimate the actual effect of the $i$-th head in layer $l$, we replace its output with distribution-matched Gaussian noise in \Eqref{equ:multi_head_output}:
\begin{equation}
\tilde{\mathbf{y}}^{(t,l)}_{(-i)}
=
\mathbf{x}^{(t,l)}
+
\mathbf{W}_O^{(l)}
\left[
\mathbf{o}^{(t,l,1)};
\cdots;
\tilde{\mathbf{o}}^{(t,l,i)};
\cdots;
\mathbf{o}^{(t,l,H)}
\right],
\end{equation}
\begin{equation}
\tilde{\mathbf{o}}^{(t,l,i)} \sim \mathcal{N}\!\left(\boldsymbol{\mu}_{t,l,i}, \boldsymbol{\Sigma}_{t,l,i}\right),
\end{equation}
where $\boldsymbol{\mu}_{t,l,i}$ and $\boldsymbol{\Sigma}_{t,l,i}$ are the corresponding mean and covariance in $\mathbf{o}^{(t,l,i)}$. Before quantifying the significance of the head, we need define the vector similarity metric as follows:
\begin{equation}
\mathrm{Sim}(\mathbf{a}, \mathbf{b})
=
\frac{1}{2}\left(1+\cos(\mathbf{a},\mathbf{b})\right)
=
\frac{1}{2}\left(1+\frac{\mathbf{a}^{\top}\mathbf{b}}{\|\mathbf{a}\|_2\,\|\mathbf{b}\|_2}\right).
\end{equation}
Thus, the information contribution of the head is 
\begin{equation}
I^{(t,l,i)} = 1 - \mathrm{Sim}\left(\mathbf{y}^{(t,l)}, \tilde{\mathbf{y}}^{(t,l)}_{(-i)}\right).
\end{equation}
After calculating the scores of all heads in layer $l$, we classify a head as \textbf{causally redundant} if 
$I(t,l,i) < \mu(I^{(t,l,:)}) - 3\sigma(I^{(t,l,:)})$,
where $\mu$ and $\sigma$ denote the mean and standard deviation operators. A head may exhibit rich internal information patterns while having negligible influence on the final output; such heads are not informative for downstream intervention. Therefore, before performing subsequent finer-grained decomposition, we first exclude causally redundant heads.
\subsection{Counterfactual Difference-in-Differences in Attention Head}
\label{did}
\begin{figure}[t]
    \centering
    \includegraphics[width=0.9\linewidth]{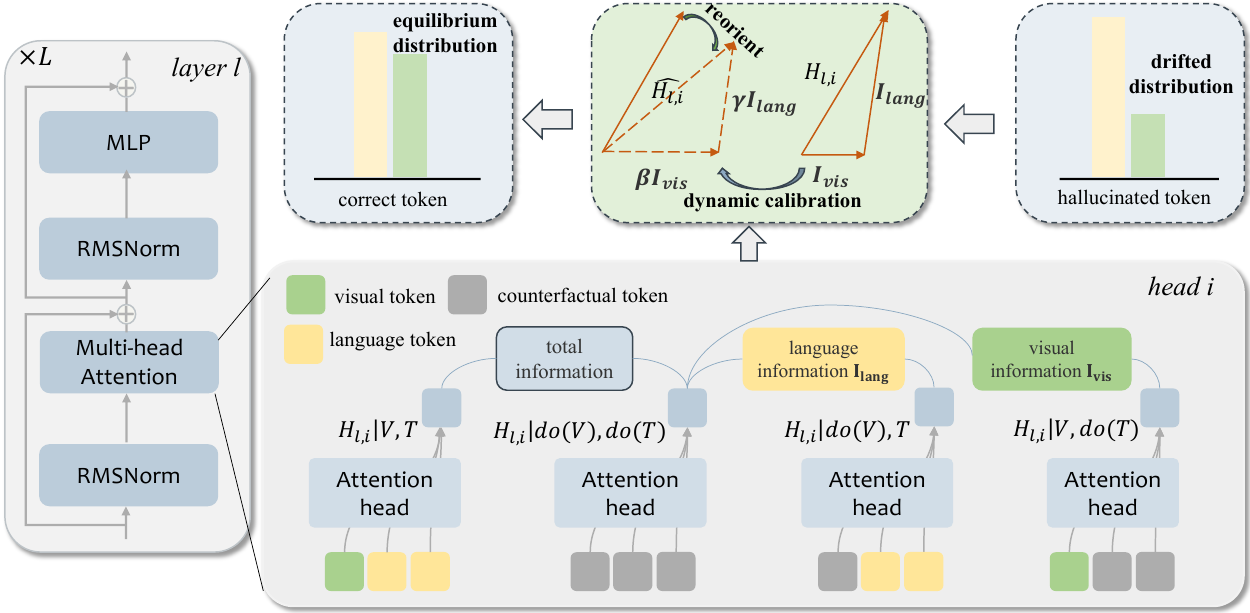}
    \caption{Illustration of the proposed HEAL. It employs the counterfactual difference-in-differences mechanism to disentangle the information distribution within attention heads and dynamically calibrates the information drift in synergy heads.}
    \label{fig:heal}
\end{figure}
Inspired by Partial Information Decomposition Theory \citep{pid_first}, for a multimodal head, its internal information can be viewed as a combination of pure visual information, pure language information, prior information and multimodal synergy, as illustrated in Figure \ref{fig:hallu_dist}(d). To practically and easily approximate these components, we construct four counterfactual inputs by independently masking visual and language tokens with Gaussian noise preserving the first- and second-order statistics:
\begin{equation}
\mathbf{H}_{11} = H_{t,l,i}(\mathbf{V}, \mathbf{T}), \quad
\mathbf{H}_{01} = H_{t,l,i}(\bar{\mathbf{V}}, \mathbf{T}), \quad
\mathbf{H}_{10} = H_{t,l,i}(\mathbf{V}, \bar{\mathbf{T}}), \quad
\mathbf{H}_{00} = H_{t,l,i}(\bar{\mathbf{V}}, \bar{\mathbf{T}}),
\end{equation}
where $\bar{\mathbf{V}}$ and $\bar{\mathbf{T}}$ denote masked visual and language tokens, respectively. We define the total information content of this head as the difference between the full and fully counterfactual settings:
\begin{equation}
\label{equ:total}
I_{\mathrm{total}}
=
\underbrace{1 - \mathrm{Sim}(\mathbf{H}_{11}, \mathbf{H}_{00})}_{\text{first-order difference}},
\end{equation}
Based on the two single-modal counterfactual cases, the visual and language information content can be calculated as
\begin{equation}
\label{equ:I_lang}
I_{\mathrm{lang}}
=
\underbrace{
(1 - \mathrm{Sim}(\mathbf{H}_{11}, \mathbf{H}_{00}))-
(1 - \mathrm{Sim}(\mathbf{H}_{11}, \mathbf{H}_{01}))}_{\text{second-order difference}}
=
\underbrace{
\mathrm{Sim}(\mathbf{H}_{11}, \mathbf{H}_{01})-\mathrm{Sim}(\mathbf{H}_{11}, \mathbf{H}_{00})}_{\text{pure language information$+$language prior}},
\end{equation}
\begin{equation}
\label{equ:I_vis}
I_{\mathrm{vis}}
=
\underbrace{
(1 - \mathrm{Sim}(\mathbf{H}_{11}, \mathbf{H}_{00}))-
(1 - \mathrm{Sim}(\mathbf{H}_{11}, \mathbf{H}_{10}))}_{\text{second-order difference}}
=
\underbrace{
\mathrm{Sim}(\mathbf{H}_{11}, \mathbf{H}_{10})-\mathrm{Sim}(\mathbf{H}_{11}, \mathbf{H}_{00})}_{\text{pure visual information$+$visual prior}}.
\end{equation}
Then the synergy effect is approximated by
\begin{equation}
\label{equ:syn}
I_{\mathrm{syn}}
=
I_{\mathrm{total}} - I_{\mathrm{vis}} - I_{\mathrm{lang}}
=
\underbrace{
1 +
\mathrm{Sim}(\mathbf{H}_{11}, \mathbf{H}_{00})
-\mathrm{Sim}(\mathbf{H}_{11}, \mathbf{H}_{01})
-\mathrm{Sim}(\mathbf{H}_{11}, \mathbf{H}_{10})}_{\text{synergy effect$-$evidence supported prior}}
\end{equation}
Note that $I_{\mathrm{syn}}$ is not the true synergy measure and may be positive or negative, since it reflects the synergy effect after discounting the overlap between prior-induced and evidence-supported information. However, our subsequent categorization primarily rely on the visual and language information content.
Based on these scores, we classify attention heads into different types. First, heads with negligible total information content are defined as \textbf{information redundant heads}:
\begin{equation}
\label{equ:info_redundant_criteria}
I_{\mathrm{total}} < \mu_{\mathrm{total}} - 3\sigma_{\mathrm{total}},
\end{equation}
where $\mu_{\mathrm{total}}$ and $\sigma_{\mathrm{total}}$ denote the mean and standard deviation of the total information scores across all heads.
For the non-redundant heads, we further distinguish modality-specific heads. If
$I_{\mathrm{vis}} > 0$ and $I_{\mathrm{lang}} \leq 0$,
the head is regarded as a visual head. Conversely, if
$I_{\mathrm{vis}} \leq 0$ and $I_{\mathrm{lang}} > 0$,
the head is regarded as a language head. When both $I_{\mathrm{vis}}$ and $I_{\mathrm{lang}}$ are positive, we compute the modality ratio: 
\begin{equation}
\alpha_{\mathrm{vis}} = \frac{I_{\mathrm{vis}}}{I_{\mathrm{vis}} + I_{\mathrm{lang}}},
\qquad
\alpha_{\mathrm{lang}} = \frac{I_{\mathrm{lang}}}{I_{\mathrm{vis}} + I_{\mathrm{lang}}}.
\end{equation}

Since the distribution of modality ratios is typically skewed\citep{vhr,sparsemm}, we first apply a Logit transformation to reduce skewness and then employ the robust median absolute deviation (MAD) to determine the thresholds:
\begin{equation}
\mathrm{MAD}(x) = \mathrm{median}\bigl(|x - \mathrm{median}(x)|\bigr).
\end{equation}
A head is classified as a visual head if
$\mathrm{Logit}(\alpha_{\mathrm{vis}}) > \mathrm{median}(\mathrm{Logit}(\alpha_{\mathrm{vis}})) + \lambda \cdot \mathrm{MAD}(\mathrm{Logit}(\alpha_{\mathrm{vis}}))$, 
and analogously for a language head. The remaining heads are categorized as synergy heads. Synergy heads can be further divided into visual-preferred and language-preferred ones according to whether $\alpha_{\mathrm{vis}} > \alpha_{\mathrm{lang}}$ or $\alpha_{\mathrm{vis}} < \alpha_{\mathrm{lang}}$. Figure \ref{fig:hallu_dist}(c) visualizes the resulting head taxonomy in a 3D space, where the three axes correspond to visual, language, and synergistic information, respectively.

\subsection{Dynamic Information Calibration}
Figure \ref{fig:heal} shows the procedure of the proposed dynamic information calibration strategy. From the above analysis, hallucinated tokens are often triggered by a significant information drift in synergy heads. Therefore, we introduce an equilibrium factor $\alpha \in (0,1)$ to characterize the desired visual-language equilibrium when generating correct tokens. Intuitively, $\alpha$ controls the model's reliance on visual information, while $1-\alpha$ controls its reliance on language information.
To reorient the information distribution in the head toward the target equilibrium $\alpha$, we need two calibration factors so that 
\begin{equation}
\frac{\beta I_{\mathrm{vis}}}{\beta I_{\mathrm{vis}} + \gamma I_{\mathrm{lang}}} = \alpha.
\end{equation}
The typical choice is 
\begin{equation}
\beta = \frac{\alpha}{\alpha_{vis}}, \qquad
\gamma = \frac{1-\alpha}{\alpha_{lang}}.
\end{equation}
After estimating the information distribution at each generation step, we dynamically calibrate the value vectors corresponding to visual and language tokens by $\beta$ and $\gamma$, respectively. Importantly, this operation is applied \emph{after} the KV cache update and \emph{before} the attention kernel, so that the internal implementation of FlashAttention \citep{flashattention2} or PagedAttention \citep{paged_attention} is not affected. The calibrated attention head can be reformulated as:
\begin{equation}
\mathbf{o}^{(t,l,i)} = \mathrm{Softmax}\left(
\frac{\mathbf{q}^{(t,l,i)} (\mathbf{K}^{(t,l,i)})^\top}{\sqrt{d}}
\right)
\begin{bmatrix} \beta \mathbf{V}^{(t,l,i)}_{vis} \\ \gamma \mathbf{V}^{(t,l,i)}_{lang} \end{bmatrix}.
\end{equation}
\begin{theorem}[Equivalence between value space and information distribution calibration]
\label{theorem_1}
Under a local additive approximation, in an attention head, calibrating the value space is equivalent to calibrating its modality-wise information distribution toward a target proportion.
\end{theorem}
\begin{theorem}[Monotonic directional effect of equilibrium factor]
\label{theorem_2}
With non-collinear visual and language components, the cosine alignment between the calibrated attention head output and visual components is strictly increasing regarding $\alpha$.
\end{theorem}
Theorem \ref{theorem_1} shows that applying the factors $\beta$ and $\gamma$ to calibrate the value vectors corresponding to visual and language tokens is equivalent to calibrating their information distribution. Theorem \ref{theorem_2} further proves that, by using the equilibrium factor to calibrate the information distribution, the geometric alignment between the attention head output and visual information changes monotonically and increasingly with respect to $\alpha$. We provide the proofs in Appendix \ref{proofs}, and the theoretical analysis of HEAL in Appendix \ref{theory_analysis}.
\subsection{Periodic, Parallelized and Batched Implementation}
\label{efficient_imple}
\textbf{Periodic Type Update.} The type distribution of heads and their internal information are inherently dynamic across generation steps. However, our empirical observations find a temporal locality: the macroscopic type distribution of heads shifts minimally within short generation windows (e.g., 10 to 15 steps). Consequently, we update the global head types periodically at a step interval $S$, while computing the information calibration factor $\beta$ and $\gamma$ at every decoding step.

\textbf{Parallelized Causal Intervention.} During the head type update, we must compute the causal noise intervention for all $H$ attention heads. To avoid sequential evaluation, we accelerate this process via parallelized tensor operations:
\begin{equation}
\widetilde{\mathbf{Y}}^{(t,l)} = \mathbf{1}_H (\mathbf{x}^{(t,l)})^\top + 
\begin{bmatrix}
\tilde{\mathbf{o}}^{(t,l,1)} & \mathbf{o}^{(t,l,2)} & \cdots & \mathbf{o}^{(t,l,H)} \\
\mathbf{o}^{(t,l,1)} & \tilde{\mathbf{o}}^{(t,l,2)} & \cdots & \mathbf{o}^{(t,l,H)} \\
\vdots & \vdots & \ddots & \vdots \\
\mathbf{o}^{(t,l,1)} & \mathbf{o}^{(t,l,2)} & \cdots & \tilde{\mathbf{o}}^{(t,l,H)}
\end{bmatrix} 
(\mathbf{W}_O^{(l)})^\top.
\end{equation}

\textbf{Batched Difference-in-Differences Analysis.} Similarly, the DiD calculation requires evaluating four distinct counterfactual states to disentangle the information distribution. Instead of computing these sequentially, we form the batched tensors $\widehat{\mathbf{q}}^{t,l,i} \in \mathbb{R}^{4 \times B \times H \times 1 \times d_h}$, $\widehat{\mathbf{K}}^{t,l,i}, \widehat{\mathbf{V}}^{t,l,i} \in \mathbb{R}^{4 \times B \times H \times L \times d_h}$ and execute a unified attention operation.
This allows the hardware accelerator to process the factual and counterfactual attention matrices simultaneously.
\section{Experiments}
\subsection{Experimental Setting}
\textbf{Baselines.} To evaluate the generalizability and effectiveness of our method, we conduct experiments on several representative MLLMs, including LLaVA series (LLaVA-1.5-7B \citep{llava_1.5} and LLaVA-NeXT-7B \citep{llavanext}), Qwen series (Qwen2-VL-7B \citep{qwen2vl}, Qwen2.5-VL-7B \citep{qwen25vl} and Qwen3-VL-8B\citep{qwen3vl}), and InternVL series (InternVL-7B \citep{internvl} and InternVL3.5-8B \citep{internvl3.5}).

\textbf{Evaluation Benchmarks.} We perform comprehensive evaluations across two primary categories of benchmarks to assess both general multimodal capabilities and specific hallucination tendencies:

(1) Comprehensive Benchmarks: we use LLaVA-Bench \citep{llava_1.5}, MME \citep{mme}, and BLINK-Twice \citep{blink_twice} to measure the impact of our method on the models' core reasoning and perception abilities.

(2) Hallucination Benchmarks: to specifically quantify hallucination reduction, we employ POPE \citep{pope} for object existence, CHAIR \citep{chair} for fine-grained image captioning, and MMHal-Bench \citep{mmhal_bench} for complex actions and spatial relationships.

\textbf{Hyperparameters.}
We use one model from a specific family to characterize the operating range of $\alpha$, and directly transfer the resulting family-level setting to the remaining models within the same or similar family without further tuning.
For LLaVA and InternVL series, we set $\alpha$ to 0.5 for simple benchmarks such as POPE and CHAIR, and to 0.6 for other challenging and comprehensive benchmarks. For Qwen families, $\alpha$ is set to 0.4 for simpler benchmarks and 0.5 for others. The update interval $S$ is consistently set to 10 generation steps across all experiments. Detailed guidelines and empirical patterns for determining these hyperparameter values are provided in Appendix \ref{para_selection}.
\begin{table}[ht]
\centering
  \caption{Comparison of HEAL with other SOTA methods on POPE, CHAIR, and MME datasets. The best performances are \textbf{bolded} and baseline model is LLaVA-1.5-7B.}
  \footnotesize
  \resizebox{1\textwidth}{!}{ 
  \begin{tabular}{l||cc|cccc|ccccc}
    \toprule
    \multirow{2}{*}{\textbf{Method}} 
    & \multicolumn{2}{c|}{\textbf{POPE}} & \multicolumn{4}{c|}{\textbf{CHAIR}} &\multicolumn{5}{c}{\textbf{MME}}   \\ 
    & \textbf{F1}$\uparrow$ & \textbf{Acc}$\uparrow$ 
    & \textbf{C$_{S}$}$\downarrow$ & \textbf{C$ _{I} $}$\downarrow$ 
    & \textbf{Recall}$\uparrow$ & \textbf{Length} 
    & \textbf{Exist.}$\uparrow$ & \textbf{Count}$\uparrow$ 
    & \textbf{Pos.}$\uparrow$ &\textbf{Color}$\uparrow$ &  \textbf{Total}$\uparrow$  \\ 
    \midrule
    Beam Search & 85.4 & 84.0 & 51.0 & 15.2 & 75.2 &102.2& 175.67 & 124.67 & 114.00 & 151.00 & 565.34 \\ 
    DoLa \citep{dola} & 80.2 & 83.1 & 57.0 & 15.2 & 78.2 & 97.5 &180.10 & 127.40 & 119.30 & 154.60 & 594.10\\
    VCD \citep{leng2024vcd} & 85.3 & 85.0 & 51.0 & 14.9 & 77.2 & 101.9& 184.66 & 137.33 & 128.67 & 153.00 & 603.66  \\
    OPERA \citep{huang2024opera} & 84.2 & 85.2 & 47.0 & 14.6 & 78.5 & 95.3 & 180.67 & 133.33 & 111.67 & 123.33 & 549.00 \\
    DOPRA \citep{dopra} & 84.6 & 84.3 & 46.3 & 13.8 & 78.2 & 96.1 & 185.67 & 138.33 & 120.67 & 133.00 & 577.67 \\
    HALC \citep{halc} & 83.9 & 84.0 & 50.2 & 12.4 & 78.4 & 97.2 & 190.00 &143.30 &128.30 & 160.00 &621.60\\
    EAH \citep{eah2024} & 85.7 & 86.0 & \textbf{36.4} & \textbf{9.9} & 74.9 & 97.7 & 190.00 & 108.33 & 145.00 & 160.66 & 603.99 \\
    SID \citep{sid} & 85.6 & 85.8 & 44.2 & 12.2 & 73.0 &99.4  & 183.90 & 132.20 & 127.80 & 155.90 & 599.80 \\
    VISTA \citep{vista} & 86.3 & 86.2 & - & - & - & - & - & - & - & - & - \\
    TAME \citep{tame} & 85.4 & 85.7 & 41.3 & 12.2 & 74.4 &98.8 &193.00 & 137.33 & 139.00 & 164.67 & 634.00 \\ 
    VAR \citep{see_what} & 86.0 & 86.5 & 52.4 & 14.5 & 79.1 & 103.0 & 190.00 & 148.33 & 138.33 & 155.00 & 631.33\\ 
    CausalLLM \citep{attention_causality} & 86.0 & 86.5& - & -& - & -& 195.00&156.00 & 135.00 &170.00 &656.00\\
    AGLA \citep{agla} & 84.6 & 85.5 & 43.0 & 14.1 & 78.9 & 98.8 & 195.00 &153.89 &129.44 & 161.67 &640.00 \\
    FarSight\citep{far_see} & - & - &41.6 & 13.2 & 75.5 & 100.6 & - & -&-&-&- \\
    MemVR \citep{look_twice} & 87.1 & 87.4 & 46.6 &13.0 & \textbf{80.8} &99.6 & 190.00 & 155.00& 133.33& 170.60&648.30 \\
    ONLY \citep{only} & 85.5 & 85.1 & 49.8 &14.3 & 75.9 &99.7 & 191.67 & 145.55& 136.66& 161.66& 635.55 \\ 
    LocoRE \citep{hallu_saliency} & 86.9 & 87.3 & 38.4 & 11.2 & 75.4 & 98.2 & 190.00 & 158.33 &133.33 & \textbf{175.00} &   656.66 \\
    VHR \citep{vhr} & 85.5 & - & 38.6 & 12.3 & - & 81.3 & - & - & - & - & - \\
    \midrule
    \rowcolor{blue!05}
    \textbf{HEAL} & $\textbf{87.7}_{\pm0.3}$ & $\textbf{88.3}_{\pm0.2}$ & $\underline{36.7}_{\pm0.4}$ & $\underline{10.7}_{\pm0.0}$ & $\underline{79.1}_{\pm0.1}$ & $99.8_{\pm0.7}$  & $\textbf{195.45}_{\pm0.23}$ & $\textbf{158.53}_{\pm0.35}$ & $\textbf{145.12}_{\pm0.48}$ & $170.66_{\pm0.66}$ & $\textbf{669.76}_{\pm1.72}$ \\
    \bottomrule
  \end{tabular}}
  \label{compare_table_hallu} 
\end{table}

\subsection{Evaluation on Hallucination Benchmarks}
As shown in Table \ref{compare_table_hallu}, existing training-free hallucination mitigation methods can be broadly categorized into two groups. The first group (OPERA \citep{huang2024opera}, DOPRA \citep{dopra}, DoLa \citep{dola} VCD \citep{leng2024vcd}, AGLA \citep{agla}, etc.) focuses on correcting the decoding process to reduce hallucinations at inference time, while the second group (TAME\citep{tame}, VAR\citep{see_what}, EAH \citep{eah2024}, VHR \citep{vhr}, FarSight\citep{far_see}, etc.) improves MLLMs' reliability by calibrating attention heads. Our method belongs to the second group, but differs from prior attention head-based approaches by explicitly disentangling the internal information composition and dynamically calibrating modality equilibrium.
On the MME and POPE benchmarks, our method achieves strong and consistent performance gains. Compared with EAH \citep{eah2024}, HEAL reaches a higher recall and longer generation length on CHAIR. We attribute this to the fact that EAH mainly strengthens certain heads, whereas our method avoids over-emphasizing one modality and instead performs an equilibrium reallocation between visual and language information.
TAME \citep{tame} aggregates token-to-token attention scores but largely overlooks the role of visual information, while VAR \citep{see_what} suppresses attention collapse by reinforcing visual information but tends to underweight textual signals. Consequently, both methods may degrade performance on long-form generation benchmarks like CHAIR. In contrast, our calibration strategy preserves the model's language fluency while improving the visual evidence in the generated output. 

\subsection{Evaluation on Comprehensive Benchmarks}
As shown in Table \ref{compare_table_hallu}, the results on the MME dataset show that HEAL consistently achieves higher scores across different evaluation categories. This suggests that our method is effective not only on hallucination-specific benchmarks, but also on a broader set of multimodal reasoning and perception tasks.
In Tables \ref{compare_mllms} and \ref{tab:recent_benchmarks}, we further integrate HEAL as a plug-and-play module into several advanced MLLMs. These results show that our method consistently improves the hallucination-related and comprehensive metrics across different architectures, which demonstrates its strong generalization.

\begin{table}[ht]
    \centering
    \caption{MLLM performance with and without HEAL on the POPE, CHAIR, and LLaVA-Bench.}
    \footnotesize
    \resizebox{0.85\textwidth}{!}{
    \setlength{\tabcolsep}{2pt}  
    \setlength{\extrarowheight}{0pt}
    \renewcommand{\arraystretch}{1.1} 
    \begin{tabular}{l| c | c c c c c}
        \midrule
        \multicolumn{1}{l|}{} & \multicolumn{1}{c|}{\textbf{Comprehensive Benchmark}} & \multicolumn{5}{c}{\textbf{Hallucination Benchmark}} \\
        \multicolumn{1}{l|}{\multirow{-2}{*}{\textbf{Method}}} & \textbf{LLaVA-Bench $\uparrow$} & \textbf{CHAIR$_{S}$ $\downarrow$} & \textbf{CHAIR$_{I}$ $\downarrow$} & \textbf{POPE-R$\uparrow$} & \textbf{POPE-F1$\uparrow$} & \textbf{POPE-A$\uparrow$} \\
        \midrule
        LLaVA-1.5-7B  & 72.5 & 51.0 & 15.2 & 87.0 & 85.4 & 84.0 \\
        \rowcolor{blue!04}
        \textbf{+ HEAL} & \textbf{75.2} & \textbf{36.9} & \textbf{10.7} & \textbf{89.1} & \textbf{87.8} & \textbf{88.5} \\
       \hline
        LLaVA-NeXT-7B  & 81.6 & 29.9 & 9.2 & 87.4 & 86.5 & 84.7 \\
        \rowcolor{blue!04}
        \textbf{+ HEAL} & \textbf{82.2} & \textbf{24.6} & \textbf{7.9} & \textbf{89.9} & \textbf{88.1} & \textbf{89.0} \\
       \hline
         Qwen2.5-VL-7B & 76.8 & 27.2 & 9.0 & 80.4 & 87.4 & 88.4 \\
        \rowcolor{blue!04}
        \textbf{+HEAL} & \textbf{78.5} & \textbf{23.3} & \textbf{8.5} & \textbf{81.3} & \textbf{88.4} & \textbf{89.1} \\  
        \hline
        Qwen2-VL-7B & 75.6 & 25.0 & 7.3 & 79.1 & 86.6 & 87.6 \\
        \rowcolor{blue!04}
        \textbf{+ HEAL } & \textbf{78.0} & \textbf{23.1} & \textbf{6.2} & \textbf{81.9} & \textbf{88.1} & \textbf{88.9} \\
        \hline
        InternVL-7B & 51.6 & 46.6 &  12.4 & 80.0 & 85.3 & 86.2 \\
        \rowcolor{blue!04}
        \textbf{+ HEAL} & \textbf{53.4} & \textbf{39.2} & \textbf{9.6} &\textbf{86.5} & \textbf{87.8} & \textbf{87.9}\\
        \midrule
    \end{tabular}}
    \label{compare_mllms}
\end{table}
\subsection{Ablation Analysis}
\label{sec_ablation}
\textbf{Update interval for head types.}
The head attribution distribution remains relatively stable within short generation windows, which motivates sparse updates instead of per-step recomputation. To quantify this effect, we vary the update interval from 2 to 20 and report the corresponding performance on the POPE dataset in Figure \ref{fig:ablation}(a). While smaller intervals achieve marginally better results, they suffer from higher computational overhead due to more counterfactual analysis. In contrast, update intervals of 10-15 maintain comparable performance while significantly improving decoding efficiency. Additional visualizations of head distribution dynamics across different generation steps are displayed in Appendix \ref{head_dist_interval}.
\begin{figure}[t]
    \centering
    \begin{minipage}{0.55\textwidth}
        \centering
        \includegraphics[width=\linewidth]{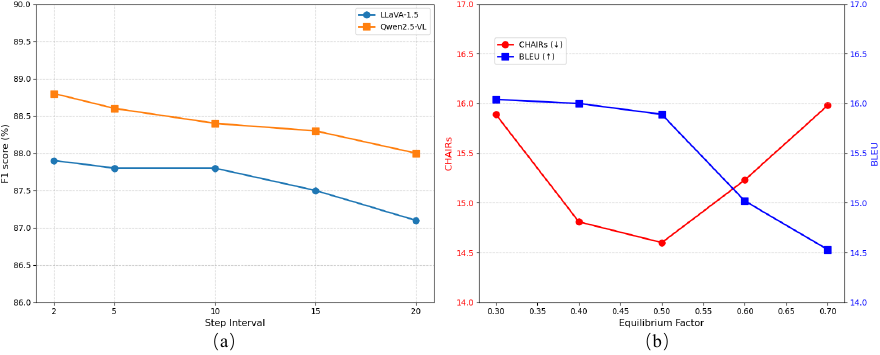}
        \caption{(a) Effect of update interval on model performance. (b) Effect of equilibrium factor on CHAIR$_s$ and BLEU (LLaVA-1.5).}
        \label{fig:ablation}
    \end{minipage}
    \hfill
    \begin{minipage}{0.4\textwidth}
        \centering
        \captionof{table}{LLaVA-Bench scores under different equilibrium factors.}
        \label{tab:ablation_llava_bench}
        \resizebox{0.85\textwidth}{!}{
        \begin{tabular}{c|cc}
            \toprule
            \textbf{Equilibrium Factor} & \textbf{LLaVA-1.5} & \textbf{Qwen2.5-VL} \\
            \midrule
            0.3 & 72.1 & 75.6 \\
            0.4 & 73.6 & 78.3 \\
            0.5 & 74.6 & 78.5 \\
            0.6 & 75.2 & 78.2 \\
            0.7 & 74.8 & 77.9 \\
            \bottomrule
        \end{tabular}}
    \end{minipage}
\end{figure}
\begin{figure}[t]
    \centering
    \includegraphics[width=1.0\linewidth]{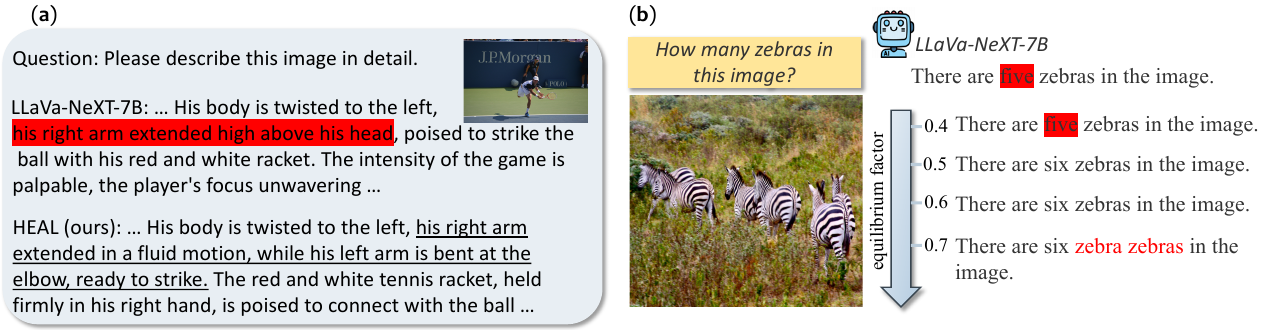}
    \caption{(a) An example of HEAL eliminating the hallucinated content. (b) The generated content under varying settings of the equilibrium factor.}
    \label{fig:demo_heal}
\end{figure}

\textbf{Equilibrium factor $\alpha$.}
We next investigate the effect of the equilibrium factor $\alpha$, which controls the trade-off between visual and language information in synergy heads. We evaluate both hallucination suppression and general generation quality under different settings of $\alpha$. As shown in Figure \ref{fig:ablation}(b), on CHAIR with LLaVA-1.5, the hallucination metric exhibits a U-shaped trend: performance first improves as $\alpha$ increases, but degrades when $\alpha$ becomes too large. This suggests that an intermediate calibration is necessary to reduce hallucination effectively; overly aggressive calibration suppresses language information and may instead harm generation quality.

We observe a similar pattern on LLaVA-Bench for both LLaVA-1.5 and Qwen2.5-VL in Table \ref{tab:ablation_llava_bench}, although the optimal value of $\alpha$ differs across models and tasks. This is expected, as different model architectures and tasks impose different demands on visual evidence. For coarse-grained tasks, limited visual information may be sufficient, whereas fine-grained tasks involving spatial reasoning typically benefit from stronger visual grounding and thus a larger $\alpha$. Moreover, models with higher-quality visual representations can often achieve good performance with a smaller equilibrium factor.

\textbf{Qualitative results.} Figure \ref{fig:demo_heal} provides qualitative examples of hallucination mitigation. In particular, Figure \ref{fig:demo_heal}(b) shows that hallucinations can only be effectively reduced once $\alpha$ reaches an appropriate range; if $\alpha$ is too large, the model may produce ungrammatical or degraded outputs, which harms the overall performance. We also note that some errors remain unresolved even with a large $\alpha$. This is likely because the relevant visual information has already been lost during visual encoding or early fusion, and such cases are better addressed by improving the model architecture or training procedure.
\subsection{Bidirectional Causal Analysis and Robustness Analysis}
\label{bidirectional_causal}
To support our claim that the information disequilibrium in synergy heads can cause hallucinations, we conduct a bidirectional intervention by directly manipulating the visual-language information ratio in Appendix \ref{bidirectional_causal_results}. Specifically, decreasing the visual information proportion can induce hallucinations in originally correct generations, whereas increasing it can alleviate hallucinations in originally hallucinated generations. These results provide stronger evidence that the visual-language information distribution is not merely correlated with hallucination behavior but can be causally intervened upon to change it.
We further conduct comprehensive robustness analyses to examine whether the identified head taxonomy and observed information drift are sensitive to different counterfactual replacements, masking strategies, measurements, and threshold selections. Across these settings, the head assignments and drift patterns remain highly consistent, supporting the stability of our interpretation. Detailed experimental settings and results are provided in Appendix~\ref{more_robustness}.
\section{Conclusion}
In this paper, we reveal that hallucinations occur when information distribution drifts away from the equilibrium in synergy heads, and propose a dynamic information calibration strategy at inference time to mitigate hallucinations. Despite its effectiveness, the calibration factors and update interval are currently determined empirically and may vary across models and tasks. Moreover, hallucinations caused by early visual encoding failures or missing visual evidence cannot always be resolved by attention head calibration alone. These limitations suggest that future work should explore more adaptive calibration policies and tighter integration with model training.

\bibliography{iclr2027_conference}
\bibliographystyle{iclr2027_conference}
\newpage
\appendix
\section{More Related Works}
\textbf{Multimodal Large Language Models.}
MLLMs extend large language models with visual perception, enabling them to encode visual inputs into language-aligned representations and generate instruction-following responses grounded in visual evidence \citep{li2023blip2,qwen3vl,internvl,llavanext,llava_1.5}.
The development of MLLMs has evolved from modular vision-language alignment to more general-purpose multimodal intelligence. Early models such as BLIP-2 \citep{li2023blip2} bridged frozen image encoders and large language models through lightweight alignment modules, while LLaVA demonstrated the effectiveness of visual instruction tuning \citep{liu2023visual}. Subsequent representative models further expanded MLLM capabilities, including MiniGPT-4~\citep{zhu2023minigpt4}, mPLUG-Owl2~\citep{ye2024mplugowl2}, LLaVA-NeXT~\citep{llavanext}, Qwen3-VL~\citep{qwen3vl}, and InternVL-3.5~\citep{internvl3.5}, showing strong performance on multimodal understanding, grounding, document parsing, etc.
Despite these advances, recent studies show that MLLMs may generate fluent but visually inconsistent outputs, including object, attribute, relation and semantic hallucinations \citep{mm_halu_det_survey,mm_hallu_mitigation_survey,pope,chair,semantic_hallu,reefknot}. These hallucinations can undermine user trust, reduce system reliability, and even lead to harmful downstream decisions in real-world applications~\citep{turner2024gpt4osystemcard}.
\section{Proofs}
\label{proofs}
\subsection{Theorem \ref{theorem_1}}
\begin{proof}
For an attention head, we can decompose its output into visual and language components by \Eqref{equ:attention}:
\begin{equation}
\label{equ:attention_decomp}
\mathbf{o}^{(t,l,i)} = [\mathbf{A}^{(t,l,i)}_{vis}, \mathbf{A}^{(t,l,i)}_{lang}]
\begin{bmatrix} \mathbf{V}^{(t,l,i)}_{vis} \\ \mathbf{V}^{(t,l,i)}_{lang} \end{bmatrix}
=\mathbf{A}^{(t,l,i)}_{vis}\mathbf{V}^{(t,l,i)}_{vis} + \mathbf{A}^{(t,l,i)}_{lang}\mathbf{V}^{(t,l,i)}_{lang} = \mathbf{h}_{\mathrm{v}} + \mathbf{h}_{\mathrm{l}},
\end{equation}
where $\mathbf{A}^{(t,l,i)}_{vis}$ and $\mathbf{A}^{(t,l,i)}_{lang}$ denote the attention weights corresponding to the visual and language tokens, respectively. Let's assume the continuous information measure $\mathcal{I}(\cdot)$ is additive and positively homogeneous, i.e.,
\begin{equation}
\mathcal{I}(\lambda \mathbf{z}) = \lambda \mathcal{I}(\mathbf{z}), \qquad \lambda \ge 0.
\end{equation}
Then we can compute the original visual-language ratio as
\begin{equation}
\alpha_{old} = \frac{\mathcal{I}(\mathbf{h}_{\mathrm{v}})}
{\mathcal{I}(\mathbf{h}_{\mathrm{v}}) + \mathcal{I}(\mathbf{h}_{\mathrm{l}})}.
\end{equation}
Given a target equilibrium factor $\alpha \in (0,1)$, we set
\begin{equation}
\beta = \frac{\alpha}{\alpha_{old}}, \qquad
\gamma = \frac{1-\alpha}{1-\alpha_{old}}.
\end{equation}
After calibrating the V vectors, we can get 
\begin{equation}
\mathbf{V}^{'}_{vis} \leftarrow \beta \mathbf{V}_{vis}, \qquad
\mathbf{V}^{'}_{lang} \leftarrow \gamma \mathbf{V}_{lang},
\end{equation}
the visual and language components in \Eqref{equ:attention_decomp} become
\begin{equation}
\label{equ:calibrated_h}
\mathbf{h}'_{\mathrm{v}} = \beta \mathbf{h}_{\mathrm{v}}, \qquad
\mathbf{h}'_{\mathrm{l}} = \gamma \mathbf{h}_{\mathrm{l}}.
\end{equation}
The new information ratio is
\begin{equation}
\begin{aligned}
\alpha_{\mathrm{new}}
&=
\frac{\mathcal{I}(\mathbf{h}'_{\mathrm{v}})}
{\mathcal{I}(\mathbf{h}'_{\mathrm{v}}) + \mathcal{I}(\mathbf{h}'_{\mathrm{l}})} \\
&=
\frac{\beta \mathcal{I}(\mathbf{h}_{\mathrm{v}})}
{\beta \mathcal{I}(\mathbf{h}_{\mathrm{v}}) + \gamma \mathcal{I}(\mathbf{h}_{\mathrm{l}})} \\
&= \alpha.
\end{aligned}
\end{equation}
Therefore, calibrating the value vectors is equivalent to steering the modality information distribution toward the target factor. In our implementation, $\mathcal{I}(\cdot)$ is instantiated by the counterfactual contribution score estimated via difference-in-differences, so the theorem holds up to a local additive approximation of the head output.
\end{proof}
\subsection{Theorem \ref{theorem_2}}
\begin{proof}
According to \Eqref{equ:attention_decomp} and \Eqref{equ:calibrated_h}, the calibrated output of an attention head can be decomposed as
\begin{equation}
\label{eq:attention_decomp_alpha}
    \mathbf{h}(\alpha)
    =
    \beta(\alpha) \mathbf{h}_{\mathrm{v}}
    +
    \gamma(\alpha) \mathbf{h}_{\mathrm{l}}.
\end{equation}
The cosine alignment between the calibrated head output and the visual component is 
\begin{equation}
    A_v(\alpha)
    =\cos\bigl(\mathbf{h}(\alpha),\mathbf{h}_{\mathrm{v}}\bigr)=
    \frac{
        \langle \mathbf{h}(\alpha),\mathbf{h}_{\mathrm{v}}\rangle
    }{
        \|\mathbf{h}(\alpha)\|_2\|\mathbf{h}_{\mathrm{v}}\|_2
    }.
    \label{eq:thm2_cos_def}
\end{equation}
Let
$
c=\frac{\langle \mathbf{h}_{\mathrm{v}},\mathbf{h}_{\mathrm{l}}\rangle}{\|\mathbf{h}_{\mathrm{v}}\|_2\|\mathbf{h}_{\mathrm{l}}\|_2},
$
then
$
-1\le c\le 1.
$
Using \Eqref{eq:attention_decomp_alpha}, the numerator of~\Eqref{eq:thm2_cos_def} becomes
\begin{equation}
    \begin{aligned}
    \langle \mathbf{h}(\alpha),\mathbf{h}_{\mathrm{v}}\rangle
    =
    \beta\|\mathbf{h}_{\mathrm{v}}\|_2^2
    +
    \gamma\langle \mathbf{h}_{\mathrm{l}},\mathbf{h}_{\mathrm{v}}\rangle
    =
    \beta\|\mathbf{h}_{\mathrm{v}}\|_2^2+\gamma \|\mathbf{h}_{\mathrm{v}}\|_2\|\mathbf{h}_{\mathrm{l}}\|_2 c.
    \end{aligned}
    \label{eq:thm2_num}
\end{equation}
Similarly,
\begin{equation}
    \begin{aligned}
    \|\mathbf{h}(\alpha)\|_2^2
    &=
    \|\beta\mathbf{h}_{\mathrm{v}} + \gamma \mathbf{h}_{\mathrm{l}}\|_2^2
    \\
    &=
    \beta^2\|\mathbf{h}_{\mathrm{v}}\|_2^2+\gamma^2\|\mathbf{h}_{\mathrm{l}}\|_2^2+2\beta\gamma \|\mathbf{h}_{\mathrm{v}}\|_2\|\mathbf{h}_{\mathrm{l}}\|_2 c.
    \end{aligned}
    \label{eq:thm2_norm}
\end{equation}
Substituting~\Eqref{eq:thm2_num} and~\Eqref{eq:thm2_norm} into~\Eqref{eq:thm2_cos_def} yields
\begin{equation}
\label{eq:thm2_cos_ab}
\begin{aligned}
    A_v(\alpha)
    & =
    \frac{
        \beta \|\mathbf{h}_{\mathrm{v}}\|_2+\gamma \|\mathbf{h}_{\mathrm{l}}\|_2 c
    }{
        \sqrt{
            \beta^2 \|\mathbf{h}_{\mathrm{v}}\|_2^2 + \gamma^2 \|\mathbf{h}_{\mathrm{l}}\|_2^2 + 2\beta\gamma \|\mathbf{h}_{\mathrm{v}}\|_2\|\mathbf{h}_{\mathrm{l}}\|_2 c
        }
    } \\
    & =
    \frac{
        1+\tau(\alpha)c
    }{
        \sqrt{
            1+\tau(\alpha)^2+2\tau(\alpha)c
        }
    },
\end{aligned}
\end{equation}
where
$
\tau(\alpha)
=
\frac{\gamma\|\mathbf{h}_{\mathrm{l}}\|_2}
{\beta\|\mathbf{h}_{\mathrm{v}}\|_2}.
$
Substituting the explicit forms of $\beta(\alpha)$ and $\gamma(\alpha)$ gives
\begin{equation}
    \frac{\partial\tau(\alpha)}{\partial\alpha}
    =
    -
    \frac{\alpha_{vis}}
    {\alpha_{lang}\alpha^2}
    \frac{\|\mathbf{h}_{\mathrm{l}}\|_2}
    {\|\mathbf{h}_{\mathrm{v}}\|_2}
    <0.
    \label{eq:thm2_tau_derivative}
\end{equation}
It remains to determine how the cosine alignment changes as $\tau$ decreases. Differentiating $A_v(\alpha)$ with respect to $\tau$ gets 
\begin{equation}    
\label{eq:thm2_f_derivative}
    \begin{aligned}
    \frac{\partial A_v(\alpha)}{\partial\tau}
    &=
    c
    \left(
        1+\tau^2+2\tau c
    \right)^{-1/2}
    \\
    &\quad
    -
    (1+\tau c)
    \left(
        1+\tau^2+2\tau c
    \right)^{-3/2}
    (\tau+c) \\
    &=
    \frac{
        c(1+\tau^2+2\tau c)
        -(1+\tau c)(\tau+c)
    }{
        \left(
            1+\tau^2+2\tau c
        \right)^{3/2}
    } \\
    &=-
    \frac{
        \tau(1-c^2)
    }{
        \left(
            1+\tau^2+2\tau c
        \right)^{3/2}
    }
    \le 0,
    \end{aligned}
\end{equation}
because
$
\tau>0, 1-c^2\ge0.
$
Combining~\Eqref{eq:thm2_tau_derivative} and~\Eqref{eq:thm2_f_derivative}, the chain rule gives
\begin{equation}
    \begin{aligned}
    \frac{\partial A_v(\alpha)}{\partial\alpha}
    &=
    \frac{\partial f(\tau)}
    {\partial\tau}
    \frac{\partial\tau(\alpha)}
    {\partial\alpha}
    \\
    &=
    \frac{
        \tau(\alpha)(1-c^2)
    }{
        \left[
            1+\tau(\alpha)^2
            +2\tau(\alpha)c
        \right]^{3/2}
    }
    \frac{\alpha_{vis}}
    {\alpha_{lang}\alpha^2}
    \frac{\|\mathbf{h}_{\mathrm{l}}\|_2}
    {\|\mathbf{h}_{\mathrm{v}}\|_2}
    \ge0.
    \end{aligned}
    \label{eq:thm2_final_derivative}
\end{equation}
Hence, the visual alignment of the calibrated head output is monotonically non-decreasing with respect to the equilibrium factor.
Moreover, because $\tau(\alpha)>0$ for $\alpha\in(0,1)$, the derivative is strictly positive whenever
$
1-c^2>0 \iff |c|<1.
$
Then, unless the visual and language components are collinear, increasing $\alpha$ strictly increases the cosine alignment with the visual component.
\paragraph{Extension to multi-head attention, residual connection, and RMSNorm \citep{rmsnorm}.}
We extend the analysis from a single attention head to a multi-head attention (MHA) layer. Let $\mathcal{H}$ denote the set of all attention heads and $\mathcal{C}\subseteq\mathcal{H}$ denote the subset of calibrated heads. For an uncalibrated head $m\notin\mathcal{C}$, $\mathbf h_m$ is independent of $\alpha$, whereas for a calibrated head $m\in\mathcal{C}$, we write
\begin{equation}
\mathbf h_m(\alpha)
=
\beta_m(\alpha)\mathbf h_{\mathrm v}^{(m)}
+
\gamma_m(\alpha)\mathbf h_{\mathrm l}^{(m)},
\end{equation}
where
$
\beta_m(\alpha)
=
\frac{\alpha}{\alpha_{\mathrm{vis}}^{(m)}},
\gamma_m(\alpha)
=
\frac{1-\alpha}{\alpha_{\mathrm{lang}}^{(m)}}.
$
Since $\mathbf W_O$ is linear, the $\alpha$-dependent part of the MHA output can be expressed as
\begin{equation}
\begin{aligned}
\mathbf h_{\mathrm{MHA}}(\alpha)
=
\mathbf h_{\mathrm{MHA}}^{(0)}
+
\sum_{m\in\mathcal C}
\mathbf W_m
\left[
\beta_m(\alpha)\mathbf h_{\mathrm v}^{(m)}
+
\gamma_m(\alpha)\mathbf h_{\mathrm l}^{(m)}
\right].
\end{aligned}
\label{eq:mha_decomposed}
\end{equation}
where $\mathbf h_{\mathrm{MHA}}^{(0)}$ collects all $\alpha$-independent contributions, including uncalibrated heads, and $\mathbf W_m$ denotes the corresponding block of the output projection matrix associated with head $m$.
We next incorporate the residual connection and the subsequent RMSNorm. Let
$
\mathbf r(\alpha)
=
\mathbf x+\mathbf h_{\mathrm{MHA}}(\alpha),
\label{eq:mha_residual}
$
and
\begin{equation}
\mathbf z(\alpha)
=
\operatorname{RMSNorm}(\mathbf r(\alpha))
=
\frac{G\mathbf r(\alpha)}
{\rho(\mathbf r(\alpha))},
\end{equation}
where
\begin{equation}
G=\operatorname{diag}(\mathbf g), \qquad
\rho(\mathbf r) = \sqrt{\frac{1}{d}\|\mathbf r\|_2^2+\epsilon}>0.
\end{equation}
Applying the fixed RMSNorm gain transformation to the residual representation gives
\begin{equation}
\begin{aligned}
G\mathbf r(\alpha)
&=
G(\mathbf x+\mathbf h_{\mathrm{MHA}}^{(0)})
+
G\left(\sum_{m\in\mathcal C} \mathbf W_m
\left[
\beta_m(\alpha)\mathbf h_{\mathrm v}^{(m)}
+
\gamma_m(\alpha)\mathbf h_{\mathrm l}^{(m)}
\right]\right) \\
&=
G(\mathbf x+\mathbf h_{\mathrm{MHA}}^{(0)})
+
\sum_{m\in\mathcal C}
\left[
\beta_m(\alpha)\bar{\mathbf h}_{\mathrm v}^{(m)}
+
\gamma_m(\alpha)\bar{\mathbf h}_{\mathrm l}^{(m)}
\right]
\end{aligned}
\label{eq:multihead_gain_residual}
\end{equation}
and
\begin{equation}
\mathbf w_{\mathrm{MHA}}
:=
\frac{\partial \left[G\mathbf r(\alpha)\right]}
{\partial\alpha}
=
\sum_{m\in\mathcal C}
\left[
\frac{\bar{\mathbf h}_{\mathrm v}^{(m)}}
{\alpha_{\mathrm{vis}}^{(m)}}
-
\frac{\bar{\mathbf h}_{\mathrm l}^{(m)}}
{\alpha_{\mathrm{lang}}^{(m)}}
\right].
\label{eq:gained_aggregate_calibration_direction}
\end{equation}
Let $\mathbf u_{\mathrm v}$ denote the normalized aggregate visual direction, then
\begin{equation}
\mathbf u_{\mathrm v}
=
\frac{\mathbf v}{\|\mathbf v\|_2},
\qquad
\mathbf v
=
\sum_{m\in\mathcal C}
\omega_m\bar{\mathbf h}_{\mathrm v}^{(m)},
\label{eq:aggregate_visual_direction}
\end{equation}
where $\omega_m\ge0$ specifies the relative contribution of each calibrated head. Since RMSNorm only introduces the positive scalar factor $\rho(\mathbf r(\alpha))^{-1}$ after the fixed gain transformation $G$, the cosine alignment with $\mathbf u_{\mathrm v}$ is
\begin{equation}
A_{\mathrm v}^{\mathrm{out}}(\alpha)
=
\cos\left(
\mathbf z(\alpha),\mathbf u_{\mathrm v}
\right)
=
\cos\left(
G\mathbf r(\alpha),\mathbf u_{\mathrm v}
\right)
=
\frac{
\langle G\mathbf{r}(\alpha),\mathbf u_{\mathrm v}\rangle
}{
\|G\mathbf{r}(\alpha)\|_2\|\mathbf u_{\mathrm v}\|_2
}. 
\label{eq:mha_final_alignment}
\end{equation}
We can decompose the gain-transformed residual representation and its $\alpha$-induced change into components parallel and orthogonal to $\mathbf u_{\mathrm v}$:
\begin{equation}
G\mathbf r(\alpha)
=
s(\alpha)\mathbf u_{\mathrm v}
+
\mathbf p(\alpha),
\qquad
\mathbf p(\alpha)\perp\mathbf u_{\mathrm v},
\label{eq:mha_y_decomposition}
\end{equation}
and
\begin{equation}
\mathbf w_{\mathrm{MHA}}
=
t\mathbf u_{\mathrm v}
+
\mathbf q,
\qquad
\mathbf q\perp\mathbf u_{\mathrm v},
\label{eq:mha_w_decomposition}
\end{equation}
where
$
s(\alpha)
=
\langle G\mathbf r(\alpha),\mathbf u_{\mathrm v}\rangle,
t
=
\langle\mathbf w_{\mathrm{MHA}},\mathbf u_{\mathrm v}\rangle.
$
So 
$
s'(\alpha)=t,
\mathbf p'(\alpha)=\mathbf q.
$
Substituting these decompositions into \Eqref{eq:mha_final_alignment} gives
\begin{equation}
A_{\mathrm v}^{\mathrm{out}}(\alpha)
=
\frac{s(\alpha)}
{\sqrt{s(\alpha)^2+\|\mathbf p(\alpha)\|_2^2}},
\end{equation}
and direct differentiation yields
\begin{equation}
\frac{\partial A_{\mathrm v}^{\mathrm{out}}(\alpha)}
{\partial\alpha}
=
\frac{
t\|\mathbf p(\alpha)\|_2^2
-
s(\alpha)
\langle\mathbf p(\alpha),\mathbf q\rangle
}{
\left(
s(\alpha)^2+\|\mathbf p(\alpha)\|_2^2
\right)^{3/2}
}.
\label{eq:mha_alignment_derivative}
\end{equation}
\Eqref{eq:mha_alignment_derivative}
$
\ge0
$
if and only if
\begin{equation}
\mathcal M_{\mathrm{MHA}}(\alpha)
=
t\|\mathbf p(\alpha)\|_2^2
-
s(\alpha)
\langle\mathbf p(\alpha),\mathbf q\rangle \ge0.
\label{eq:mha_residual_margin}
\end{equation}
Importantly, \Eqref{eq:mha_residual_margin} is imposed on the aggregate MHA representation rather than on individual attention heads. The projection matrix $W_O$ may rotate and mix the contributions from different heads, so head-wise monotonicity is neither necessary nor sufficient for monotonicity of the final multi-head output. Instead, all calibrated heads jointly determine the effective direction $\mathbf w_{\mathrm{MHA}}$ in~\Eqref{eq:gained_aggregate_calibration_direction}.

The margin $\mathcal M_{\mathrm{MHA}}(\alpha)$ has a simple geometric interpretation. The term $t\|\mathbf p(\alpha)\|_2^2$ measures the component of the aggregate MHA-induced motion that increases alignment with the visual direction, whereas $s(\alpha)\langle\mathbf p(\alpha),\mathbf q\rangle$ captures the opposing angular perturbation caused by the orthogonal component of the residual representation. Hence, $\mathcal M_{\mathrm{MHA}}(\alpha)\ge0$ characterizes the regime in which the aggregate visual-directed motion dominates the residual-induced angular perturbation. 
The above analysis establishes monotonicity of the final visual alignment after the complete sequence of multi-head aggregation, residual connection, and RMSNorm.
\end{proof}
\section{Theoretical Analysis of HEAL}
\label{theory_analysis}
\subsection{Theoretical Contributions}
Our theoretical results include:
\begin{itemize}[leftmargin=*, itemsep=3pt, topsep=0pt, parsep=0pt, partopsep=0pt]
    \item First, value vector calibration is equivalent to adjusting the information distribution between visual and language components within synergy heads (Theorem \ref{theorem_1}).
    \item Second, the calibration changes the representation geometry predictably, rather than acting as an empirical heuristic (Theorem \ref{theorem_2}).
\end{itemize}
Beyond these theorems, HEAL also provides a potential information decomposition framework for Partial Information Decomposition (PID) Theory \citep{pid_first,pird,broadcast_pid}, where multimodal information can be interpreted as visual-specific, language-specific, shared (synergistic), and redundant components. Strict PID decomposition generally requires strong mathematical formulations that are difficult to directly instantiate in modern neural representations. HEAL provides a practical decomposition instantiation for multimodal transformers.
\subsection{Discussions between PID Theory and HEAL}
Our decomposition framework is closely related to the classical PID theory \citep{pid_first}, which characterizes the information that multiple sources provide about a target in terms of redundant, unique, and synergistic components.
Let \(X_V\) and \(X_L\) denote the visual and language sources, and \(Y\) a target variable associated with the information represented by a multimodal attention head. PID decomposes the joint mutual information as
\begin{equation}
\begin{aligned}
I(Y;X_V,X_L)
&=
\int_{\mathcal Y}
\int_{\mathcal X_V}
\int_{\mathcal X_L}
p(y,x_V,x_L)
\log
\frac{
p(y,x_V,x_L)
}{
p(y)p(x_V,x_L)
}
\,\mathrm d x_L\,\mathrm d x_V\,\mathrm d y \\
&=
R(Y;X_V,X_L) + U_V(Y;X_V|X_L) + U_L(Y;X_L|X_V) + S(Y;X_V,X_L).
\end{aligned}
\label{eq:pid_joint_mi}
\end{equation}
where \(R\) denotes redundant information shared by the two sources, \(U_V\) and \(U_L\) denote source-specific information, and \(S\) denotes information that emerges only from the joint availability of both sources. Equivalently, using conditional entropy,
\begin{equation}
\begin{aligned}
I(Y;X_V,X_L)
&=
H(Y)-H(Y\mid X_V,X_L)
\\
&=
\int p(y,x_V,x_L)
\log
\frac{
p(y\mid x_V,x_L)
}{
p(y)
}
\,\mathrm d y\,\mathrm d x_V\,\mathrm d x_L .
\end{aligned}
\label{eq:pid_joint_entropy}
\end{equation}
\(I(Y;X_V,X_L)\) measures the uncertainty reduction about \(Y\) obtained by jointly observing the visual and language sources. In PID theory, marginal mutual information contains both redundant and unique information, whereas conditional mutual information contains unique and synergistic information. Thus, we have
\begin{equation}
I(Y;X_V)
=
\int_{\mathcal Y}
\int_{\mathcal X_V}
p(y,x_V)
\log
\frac{
p(y\mid x_V)
}{
p(y)
}
\,\mathrm d x_V\,\mathrm d y = R+U_V,
\label{eq:pid_visual_mi}
\end{equation}
\begin{equation}
I(Y;X_L)
=
\int_{\mathcal Y}
\int_{\mathcal X_L}
p(y,x_L)
\log
\frac{
p(y\mid x_L)
}{
p(y)
}
\,\mathrm d x_L\,\mathrm d y = R+U_L,
\label{eq:pid_language_mi}
\end{equation}
\begin{equation}
\begin{aligned}
I(Y;X_V\mid X_L)
&=
\int
p(y,x_V,x_L)
\log
\frac{
p(y\mid x_V,x_L)
}{
p(y\mid x_L)
}
\,\mathrm d y\,\mathrm d x_V\,\mathrm d x_L
\\
&=
H(Y\mid X_L)-H(Y\mid X_V,X_L) = U_V+S,
\end{aligned}
\label{eq:pid_cond_visual}
\end{equation}
\begin{equation}
\begin{aligned}
I(Y;X_L\mid X_V)
&=
\int
p(y,x_V,x_L)
\log
\frac{
p(y\mid x_V,x_L)
}{
p(y\mid x_V)
}
\,\mathrm d y\,\mathrm d x_V\,\mathrm d x_L
\\
&=
H(Y\mid X_V)-H(Y\mid X_V,X_L) = U_L+S.
\end{aligned}
\label{eq:pid_cond_language}
\end{equation}

Importantly, these mutual information quantities in \ref{eq:pid_visual_mi}-\ref{eq:pid_cond_language} provide only three independent equations for the four unknown PID atoms. Therefore, the redundancy term \(R\) cannot be uniquely determined from Shannon mutual information alone. A specific redundancy function must be introduced \citep{qui,measuring_pid}.
For example, under the original \(I_{\min}\) construction of Williams and Beer \citep{pid_first}, the pointwise specific information of a source \(X\) about an outcome \(y\) is
\begin{equation}
I_{\mathrm{spec}}(y;X)
=
D_{\mathrm{KL}}
\left[
p(x\mid y)
\middle\|
p(x)
\right]
=
\int
p(x\mid y)
\log
\frac{
p(x\mid y)
}{
p(x)
}
\,\mathrm d x .
\label{eq:specific_information}
\end{equation}
The redundancy between the visual and language sources is then
\begin{equation}
R_{\min}(Y;X_V,X_L)
=
\int_{\mathcal Y}
p(y)
\min
\left\{
I_{\mathrm{spec}}(y;X_V),
I_{\mathrm{spec}}(y;X_L)
\right\}
\,\mathrm d y.
\label{eq:imin_continuous}
\end{equation}
Once a redundancy function \(R^\star\) is selected, the remaining PID atoms follow algebraically:
\[
U_V^\star
=
I(Y;X_V)-R^\star,\\
U_L^\star
=
I(Y;X_L)-R^\star,
\]
\begin{align}
S^\star
&=
I(Y;X_V,X_L)
-I(Y;X_V)
-I(Y;X_L)
+R^\star.
\label{eq:pid_synergy_general}
\end{align}
\Eqref{eq:pid_synergy_general} is particularly relevant to \Eqref{equ:syn} in HEAL. It makes explicit that the synergistic atom is not simply the difference between joint and marginal information: the redundancy term must also be restored.
Different choices of \(R^\star\) lead to different PID decompositions. For example, BROJA defines unique and synergistic information through an optimization over distributions having fixed source-target marginals \citep{qui}, while the CCS approach defines redundancy by identifying common pointwise changes in surprisal \citep{measuring_pid}. The existence of these alternatives emphasizes that PID provides a decomposition principle rather than a unique redundancy estimator.

HEAL does not estimate \(p(y,x_V,x_L)\) explicitly. Instead, it observes the change induced in an attention head representation after selectively removing visual and language inputs. For two hidden states \(a\) and \(b\), HEAL defines the representation discrepancy as
\begin{equation}
\begin{aligned}
D(a,b)
&=
1-\operatorname{Sim}(a,b)
\\
&=
\frac{1}{2}
\left[
1-
\frac{
\langle a,b\rangle
}{
\|a\|_2\|b\|_2
}
\right].
\end{aligned}
\label{eq:heal_discrepancy_pid}
\end{equation}
Rather than interpreting \(D(a,b)\) as a Shannon mutual information, we regard it as a task-relevant measure of the information contribution associated with the corresponding counterfactual intervention. In fact, for a distribution of inputs \((V,T)\), the expected counterfactual information is written as
\begin{equation}
\begin{aligned}
\mathcal I
&=
\mathbb E_{(V,T)\sim p(V,T)}
\left[
D
\left(
H(V,T),
H^{\mathrm{cf}}(V,T)
\right)
\right] \\
&=
\int_{\mathcal V}
\int_{\mathcal T}
p(v,t)
\frac{1}{2}
\left[
1-
\frac{
\left\langle
H(v,t),
H^{\mathrm{cf}}(v,t)
\right\rangle
}{
\|H(v,t)\|_2
\|H^{\mathrm{cf}}(v,t)\|_2
}
\right]
\,\mathrm d v\,\mathrm d t.
\label{eq:heal_expected_information}
\end{aligned}
\end{equation}
This expectation-level formulation makes clear the analogy with mutual information: Shannon information averages a pointwise log-likelihood ratio over samples, whereas HEAL averages a pointwise representation discrepancy over counterfactual interventions. The two quantities have different statistical definitions, but they share the same operational structure of measuring information contribution through changes induced by conditioning or intervention.

In HEAL framework (Section \ref{did}), we define
\begin{equation}
D_{00}
=
D(H_{11},H_{00}),
\qquad
D_{10}
=
D(H_{11},H_{10}),
\qquad
D_{01}
=
D(H_{11},H_{01}).
\label{eq:heal_D_notation}
\end{equation}
Thus, the total information contribution of the two modalities is
\begin{equation}
I_{\mathrm{total}}
=
D_{00}
=
\frac{1}{2}
\left[
1-
\frac{
\langle H_{11},H_{00}\rangle
}{
\|H_{11}\|_2\|H_{00}\|_2
}
\right].
\label{eq:heal_total_expanded}
\end{equation}
The visual contribution is obtained by subtracting the discrepancy that remains after masking language:
\begin{equation}
\begin{aligned}
I_{\mathrm{vis}}
&=
D_{00}-D_{10}
=
\frac{1}{2}
\left[
\frac{
\langle H_{11},H_{10}\rangle
}{
\|H_{11}\|_2\|H_{10}\|_2
}
-
\frac{
\langle H_{11},H_{00}\rangle
}{
\|H_{11}\|_2\|H_{00}\|_2
}
\right].
\end{aligned}
\label{eq:heal_visual_expanded}
\end{equation}
Similarly,
\begin{equation}
\begin{aligned}
I_{\mathrm{lang}}
&=
D_{00}-D_{01}
=
\frac{1}{2}
\left[
\frac{
\langle H_{11},H_{01}\rangle
}{
\|H_{11}\|_2\|H_{01}\|_2
}
-
\frac{
\langle H_{11},H_{00}\rangle
}{
\|H_{11}\|_2\|H_{00}\|_2
}
\right].
\end{aligned}
\label{eq:heal_language_expanded}
\end{equation}
The remaining term is obtained by
\begin{equation}
\begin{aligned}
I_{\mathrm{syn}}
&=
I_{\mathrm{total}}
-
I_{\mathrm{vis}}
-
I_{\mathrm{lang}}
\\
&=
D_{00}
-
(D_{00}-D_{10})
-
(D_{00}-D_{01})
\\
&=
D_{10}+D_{01}-D_{00}.
\end{aligned}
\label{eq:heal_syn_expanded}
\end{equation}
The relationship to PID becomes explicit under an information-faithfulness assumption. Suppose that, in expectation over the data distribution, the counterfactual representation discrepancies approximate the corresponding Shannon information quantities:
\begin{equation}
\begin{aligned}
\mathbb E[I_{\mathrm{total}}]
&\approx
I(Y;X_V,X_L),
\\
\mathbb E[I_{\mathrm{vis}}]
&\approx
I(Y;X_V),
\\
\mathbb E[I_{\mathrm{lang}}]
&\approx
I(Y;X_L).
\end{aligned}
\label{eq:heal_pid_faithfulness}
\end{equation}
Substituting the PID identities into the HEAL interaction term gives
\begin{equation}
\begin{aligned}
I_{\mathrm{syn}}
&\approx
I(Y;X_V,X_L)
-
I(Y;X_V)
-
I(Y;X_L)
\\
&=
(R+U_V+U_L+S)
-
(R+U_V)
-
(R+U_L)
\\
&=
S-R.
\end{aligned}
\label{eq:heal_syn_pid_derivation}
\end{equation}
At the same time,
\begin{equation}
\begin{aligned}
I_{\mathrm{vis}}
&\approx
R+U_V,
\\
I_{\mathrm{lang}}
&\approx
R+U_L,
\\
I_{\mathrm{total}}
&\approx
R+U_V+U_L+S.
\end{aligned}
\end{equation}
Thus, the modality-specific scores of HEAL should be understood as modality-attributable information rather than strictly unique information. Let \(R^\star\) be any admissible redundancy function; then
\begin{equation}
\boxed{
\begin{aligned}
U_V^\star
&\approx
I_{\mathrm{vis}}-R^\star,
\\
U_L^\star
&\approx
I_{\mathrm{lang}}-R^\star,
\\
S^\star
&\approx
I_{\mathrm{syn}}+R^\star.
\end{aligned}}
\label{eq:heal_general_pid_reconstruction}
\end{equation}
Indeed,
\begin{equation}
\begin{aligned}
R^\star
+
U_V^\star
+
U_L^\star
+
S^\star
&\approx
R^\star
+
(I_{\mathrm{vis}}-R^\star)
+
(I_{\mathrm{lang}}-R^\star)
+
(I_{\mathrm{syn}}+R^\star)
\\
&=
I_{\mathrm{vis}}
+
I_{\mathrm{lang}}
+
I_{\mathrm{syn}}
\\
&=
I_{\mathrm{total}},
\end{aligned}
\label{eq:heal_pid_exact_reconstruction}
\end{equation}
which demonstrates that an explicit redundancy estimate provides the missing degree of freedom required to transform the HEAL decomposition into a full PID decomposition.

The above derivation also explains a characteristic property of the HEAL synergy score. In strict PID, the synergistic atom \(S\) is defined after explicitly separating the redundancy \(R\). In HEAL, this redundancy is not separately estimated and is implicitly counted in both \(I_{\mathrm{vis}}\) and \(I_{\mathrm{lang}}\). As a result,
$
I_{\mathrm{syn}}^{\mathrm{HEAL}} \approx S-R,
$
rather than \(S\) itself.
Therefore,
\begin{equation}
\boxed{
\begin{aligned}
I_{\mathrm{syn}}>0
&\Longleftrightarrow
S>R,
\\
I_{\mathrm{syn}}=0
&\Longleftrightarrow
S=R,
\\
I_{\mathrm{syn}}<0
&\Longleftrightarrow
S<R.
\label{eq:heal_signed_synergy}
\end{aligned}}
\end{equation}
A negative HEAL synergy score should not be interpreted as negative Shannon synergy. It indicates that shared information is sufficiently large to dominate the net joint interaction. This is also why the signed nature of \(I_{\mathrm{syn}}\) is not inconsistent with the non-negativity of the atoms in the classical PID formulation. The original PID was introduced partly to separate the non-negative synergy and redundancy atoms that become conflated in conventional interaction information.

The preceding results suggest a natural two-stage interpretation of HEAL. First, the counterfactual Difference-in-Differences procedure identifies modality-attributable information directly in the representation space:
\begin{equation}
\left\{
H_{11},H_{10},H_{01},H_{00}
\right\}
\longrightarrow
\left\{
I_{\mathrm{vis}},
I_{\mathrm{lang}},
I_{\mathrm{syn}}
\right\}.
\end{equation}
Second, if a redundancy functional \(R^\star\) is estimated from the underlying joint distribution, the corresponding PID atoms can be obtained through
\begin{equation}
\left\{
I_{\mathrm{vis}},
I_{\mathrm{lang}},
I_{\mathrm{syn}},
R^\star
\right\}
\longrightarrow
\left\{
U_V^\star,
U_L^\star,
R^\star,
S^\star
\right\}.
\end{equation}
This establishes HEAL as a counterfactual, representation-level analogue of PID rather than an exact replacement for a probability-based PID estimator. The distinction is important: the HEAL discrepancy in \Eqref{eq:heal_discrepancy_pid} is based on cosine geometry of hidden representations, whereas Shannon mutual information is based on a log-density ratio. Nevertheless, both constructions quantify the contribution of a source through the change in information available about a target, and the counterfactual factorial design in HEAL provides a natural operational analogue of marginal and joint source access.
This connection is also complementary to existing multimodal information decomposition approaches, which estimate redundancy, uniqueness, and synergy primarily at the modality level \citep{mllm_pid,quantifying_mllm}. In contrast, HEAL applies the decomposition locally to individual attention heads and generation steps, allowing modality interaction to be analyzed and subsequently intervened upon at the level of internal computation.
\section{More Results and Visualization}
\subsection{Bidirectional Causal Analysis}
\label{bidirectional_causal_results}
\label{}
\begin{table}[ht]
    \centering
    \caption{Bidirectional causal intervention in LLaVA-1.5-7B and Qwen2.5-VL-7B by manipulating the visual-language ratio. Performance scores are evaluated on the LLaVA-Bench.}
    \label{tab:causal_ratio}
    \footnotesize
    \resizebox{0.6\textwidth}{!}{
    \setlength{\tabcolsep}{2pt}  
    \setlength{\extrarowheight}{0pt}
    \renewcommand{\arraystretch}{1.0} 
    \begin{tabular}{lcccccc}
        \toprule
        \textbf{Ratio}
        & \textbf{0.3}
        & \textbf{base (no intervening)}
        & \textbf{0.4}
        & \textbf{0.5}
        & \textbf{0.6}
        & \textbf{0.7} \\
        \midrule
        LLaVA-1.5-7B & 72.1 & 72.5 & 73.6 & 74.6 & 75.2 & 74.8 \\
        \addlinespace
        Qwen2.5-VL-7B & 75.6 & 76.8 & 78.3 & 78.5 & 78.2 & 77.9 \\
        \bottomrule
    \end{tabular}}
\end{table}
\begin{figure}[h]
    \centering
    \includegraphics[width=0.9\linewidth]{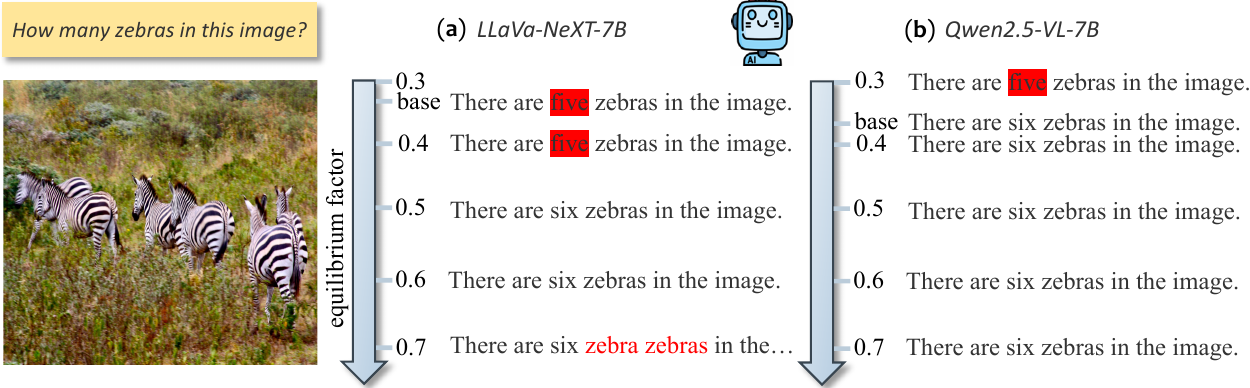}
    \caption{Bidirectional causal intervention in LLaVA-1.5-7B and Qwen2.5-VL-7B replies by manipulating the visual-language information distribution (equilibrium factor). \textit{base} means no intervention.}
    \label{fig:bidirectional_causal}
\end{figure}
We conduct a bidirectional causal intervention analysis by directly manipulating the visual-language information ratio. As shown in Figure \ref{fig:bidirectional_causal} and Table \ref{tab:causal_ratio}, the results demonstrate two causal directions:
\begin{itemize}[leftmargin=*, itemsep=0pt, topsep=0pt, parsep=0pt, partopsep=0pt]
    \item When the model originally produces a correct response, artificially decreasing the visual information proportion can induce hallucination.
    \item When the model originally produces hallucination, increasing the visual information proportion can alleviate hallucination.
\end{itemize}
Therefore, the observed relationship between hallucination and visual-language disequilibrium is not merely correlational. The intervention results show that manipulating this distribution leads to predictable changes in hallucination behavior, providing stronger evidence for our causal claim.
\subsection{Robustness Results and Details}
\label{more_robustness}
The credibility of our interpretation depends on whether the proposed head taxonomy and the observed drift remain stable under different design choices. To support the central claim, we conduct a comprehensive robustness analysis.

\textbf{Experiment setting}. We sampled 50 image-caption pairs generated by LLaVA-Next-7B and Qwen2.5-VL-7B on the CHAIR benchmark, covering both correct and various hallucinated tokens (approximately 9,000 tokens in total). For each token, every attention head is assigned a label, and we further computed the visual-language information ratio within the identified synergy heads for both correct and hallucinated tokens. The reported Head Assignment Agreement is defined as the percentage of head labels that remain identical to those obtained using the default setting across all tokens.

\begin{table}[ht]
    \centering
    \caption{Head type assignment agreement under different masking strategies.}
    \label{tab:masking_agreement}
    \footnotesize
    \resizebox{1.0\textwidth}{!}{
    \setlength{\tabcolsep}{2pt}  
    \setlength{\extrarowheight}{0pt}
    \renewcommand{\arraystretch}{1.0} 
    \begin{tabular}{lcccc}
        \toprule
        \textbf{Masking strategy}
        & \textbf{Gaussian masking (base)}
        & \textbf{one-point (zero) masking}
        & \textbf{uniform masking}
        & \textbf{swapping actual tokens} \\
        \midrule
        Head Assignment Agreement (\%)
        & 100.00
        & 92.13
        & 93.21
        & 95.36 \\
        \bottomrule
    \end{tabular}}
\end{table}
\begin{table}[ht]
    \centering
    \caption{Robustness to different masking strategies. We report the visual-language information ratios of correct and hallucinated tokens in synergy heads and LLaVA-1.5-7B performance on POPE.}
    \label{tab:masking_robust}
    \footnotesize
    \resizebox{1.0\textwidth}{!}{
    \setlength{\tabcolsep}{2pt}  
    \setlength{\extrarowheight}{0pt}
    \renewcommand{\arraystretch}{1.1} 
    \begin{tabular}{lcccc}
        \toprule
        \textbf{Masking strategy}
        & \textbf{Gaussian masking (base)}
        & \textbf{one-point (zero) masking}
        & \textbf{uniform masking}
        & \textbf{swapping actual tokens} \\
        \midrule
        F1 score $\uparrow$
        & 87.84
        & 87.12
        & 86.93
        & 87.53 \\
        \hdashline
        correct tokens
        & 0.43:0.54
        & 0.41:0.53
        & 0.39:0.48
        & 0.41:0.56 \\
        hallucinated tokens
        & 0.28:0.62
        & 0.21:0.52
        & 0.23:0.57
        & 0.31:0.66 \\
        \bottomrule
    \end{tabular}}
\end{table}
\textbf{(1) Robustness to masking strategies.}
To validate whether our conclusions depend on the specific masking strategy, we further consider zero masking, uniform masking, and swapping the actual visual/language tokens from another image or text. As shown in Table~\ref{tab:masking_agreement}, the resulting head taxonomy is highly similar across all masking strategies, achieving 92.13\%-95.36\% head assignment agreement. 
More importantly, from Table~\ref{tab:masking_robust}, the central claim that the visual-language information drift between correct and hallucinated tokens, remains consistently observable under every masking strategy. Additionally, all masking strategies yield comparable downstream performance on POPE, with Gaussian masking achieving the best. These results indicate that our taxonomy and the observed drift are properties of the model itself rather than artifacts of a particular masking strategy.

\begin{table}[ht]
    \centering
    \caption{Robustness to threshold selection. We report the visual-language
    information ratios of correct and hallucinated tokens in synergy heads
    and LLaVA-1.5-7B performance on POPE.}
    \label{tab:threshold_ablation}
    \begin{subtable}[t]{0.48\columnwidth}
        \centering
        \caption{$\sigma_{total}$ coefficient in \Eqref{equ:info_redundant_criteria}}
        \label{tab:sigma_total_ablation}
        \scriptsize
        \setlength{\tabcolsep}{2pt}
        \renewcommand{\arraystretch}{1.1}
        \begin{tabular}{lcccc}
            \toprule
            $\sigma_{total}$ coefficient
            & \textbf{1}
            & \textbf{2}
            & \textbf{3}
            & $\mathbf{\infty}$ \\
            \midrule
            correct tokens
            & 0.45:0.51
            & 0.40:0.49
            & 0.43:0.54
            & 0.45:0.52 \\
            hallucinated tokens
            & 0.24:0.63
            & 0.23:0.61
            & 0.28:0.62
            & 0.29:0.66 \\
            \hdashline
            F1 score $\uparrow$
            & 87.06
            & 87.43
            & \textbf{87.84}
            & 87.65 \\
            \bottomrule
        \end{tabular}
    \end{subtable}
    \hfill
    \begin{subtable}[t]{0.48\columnwidth}
        \centering
        \caption{MAD coefficient}
        \label{tab:mad_ablation}
        \scriptsize
        \setlength{\tabcolsep}{2pt}
        \renewcommand{\arraystretch}{1.1}
        \begin{tabular}{lcccc}
            \toprule
            MAD coefficient ($\lambda$)
            & \textbf{1.4826}
            & \textbf{2.9652}
            & \textbf{4.4478}
            & $\mathbf{\infty}$ \\
            \midrule
            correct tokens
            & 0.43:0.57
            & 0.43:0.54
            & 0.47:0.52
            & 0.38:0.49 \\
            hallucinated tokens
            & 0.19:0.56
            & 0.28:0.62
            & 0.20:0.66
            & 0.23:0.65 \\
            \hdashline
            F1 score $\uparrow$
            & 86.95
            & \textbf{87.84}
            & 87.51
            & 86.23 \\
            \bottomrule
        \end{tabular}
    \end{subtable}
    \vspace{6pt}
    \parbox{0.98\columnwidth}{%
    \scriptsize
    \textit{Note.}
    We use MAD to robustly estimate dispersion, where $1.4826 \approx 1/\Phi^{-1}(0.75)$ is the standard consistency factor under normality. Thus, 1.4826, 2.9652, and 4.4478 correspond approximately to the conventional $1\sigma$, $2\sigma$, and $3\sigma$ thresholds, respectively.
    }
\end{table}
\textbf{(2) Robustness to threshold selection.}
The $3\sigma_{total}$ criterion and $\lambda\cdot$MAD modality-ratio threshold indeed affect the distribution of head types because they explicitly define the classification boundaries.
As shown in Tables~\ref{tab:sigma_total_ablation} and \ref{tab:mad_ablation}, although the proportion of heads assigned to each category changes moderately, the fundamental observation remains unchanged: hallucinated tokens consistently exhibit a significant distribution shift toward language representations in synergy heads. The POPE performance is also stable, with the best performance achieved around the standard $3\sigma_{total}$ information redundancy criterion and $\lambda=2.9652$. Overall, these experiments show that while exact head assignments naturally vary with different decision boundaries, the existence of synergy heads, the distribution disequilibrium, and the effectiveness of HEAL remain remarkably stable.

\begin{table}[ht]
    \centering
    \caption{Head assignment agreement under different measurement metrics.}
    \label{tab:metric_agreement}
    \footnotesize
    \resizebox{1.0\textwidth}{!}{
    \setlength{\tabcolsep}{4pt}  
    \setlength{\extrarowheight}{0pt}
    \renewcommand{\arraystretch}{1.1} 
    \begin{tabular}{lc|c|c}
        \toprule
        \textbf{Metric}
        & \textbf{normalized cosine similarity}
        & \textbf{$L_2$ norm + Sigmoid}
        & \textbf{cosine similarity + $L_2$ norm + Sigmoid} \\
        \midrule
        Head Assignment Agreement (\%)
        & 100.00 (base)
        & 95.73
        & 97.98 \\
        \bottomrule
    \end{tabular}}
\end{table}
\textbf{(3) Robustness to measurement metrics.}
We further evaluate alternative head scoring metrics, including normalized $L_2$ norm and a hybrid cosine$+L_2$ metric. In Table \ref{tab:metric_agreement}, all metrics produce highly similar head assignments, suggesting that the proposed taxonomy is largely independent of the specific metric.
We adopt cosine similarity because it is most consistent with the geometric interpretation of HEAL. As illustrated in Figure \ref{fig:heal} and theoretically supported by Theorem \ref{theorem_2}, HEAL performs a directional calibration of the output representation by adjusting the equilibrium between visual and language components. Cosine similarity directly measures this angular change, providing a more interpretable geometric characterization than magnitude-based metrics such as $L_2$ norm.

\begin{table}[ht]
    \centering
    \caption{Robustness to different replacement distributions in Section \ref{causal_intervention}. We report the Head Assignment Agreement metric and LLaVA-1.5-7B performance on POPE.}
    \label{tab:replacement_robust}
    \scriptsize
    \resizebox{1.0\textwidth}{!}{
    \setlength{\tabcolsep}{4pt}  
    \setlength{\extrarowheight}{0pt}
    \renewcommand{\arraystretch}{1.2}
    \begin{tabular}{l|c|c|c|c|c|c}
        \toprule
        & \textbf{no intervention}
        & \textbf{Gaussian (base)}
        & \textbf{one-point (zero)}
        & \textbf{uniform}
        & \textbf{Cauchy (extreme outliers)}
        & \textbf{swapping actual outputs} \\
        \midrule
        Head Assignment Agreement (\%)
        & -
        & 100.00
        & 95.25
        & 96.31
        & 97.46
        & 96.97 \\
        \hdashline
        F1 score $\uparrow$
        & 86.89
        & 87.84
        & 87.73
        & 87.69
        & 87.81
        & 87.75 \\
        \bottomrule
    \end{tabular}}
    
    \vspace{2pt}
    \parbox{0.96\columnwidth}{%
    \scriptsize
    \textit{Note.}
    Every attention head is only assigned a binary label (redundant or non-redundant) to compute Head Assignment Agreement.}
\end{table}
\textbf{(4) Robustness to replacement distributions.}
In Section \ref{causal_intervention}, our method does not assume that attention head outputs follow a Gaussian distribution. The Gaussian distribution is only used as a replacement noise during the causal intervention, rather than as a probabilistic model of the underlying activations. The objective of this intervention is simply to remove the instance-specific information carried by a head so that its causal contribution can be estimated.

To verify whether the choice of replacement affects causally redundant head identification, we further replace Gaussian noise with several different alternatives, including a one-point distribution (zero replacement), a uniform distribution, a heavy-tailed Cauchy distribution containing extreme outliers, and swapping the head outputs from another sample. As reported in Table~\ref{tab:replacement_robust}, all replacements have highly consistent classifications. These results indicate that redundant head identification is largely distribution-independent, and our causal intervention is robust even under heavy-tailed perturbations or real activation replacement.
More importantly, the hallucination mitigation performance also remains stable. Replacing the Gaussian distribution with these alternatives produces very similar POPE F1 scores, while all intervention strategies consistently outperform the variant without causal intervention. This further demonstrates that the success of this stage stems from the intervention itself rather than the specific choice of the replacement distribution.

Regarding swapping the actual head outputs, it may provide a meaningful alternative for estimating head importance. However, such a strategy requires an additional reference sample during inference, making it impractical for the inference-time setting. In contrast, Gaussian replacement is computationally efficient, low-order moment-preserving, and can be performed directly on the current input without introducing external samples.

\begin{table}[ht]
    \centering
    \caption{Redundant head assignment agreement before and after
    Yeo-Johnson transformation.}
    \label{tab:yeo_johnson}
    \scriptsize
    \resizebox{1.0\textwidth}{!}{
    \setlength{\tabcolsep}{4pt}  
    \setlength{\extrarowheight}{0pt}
    \renewcommand{\arraystretch}{1.1}
    \begin{tabular}{lcc}
        \toprule
        Metric
        & \textbf{before Yeo-Johnson transformation \textbf{(base)}}
        & \textbf{after Yeo-Johnson transformation}
        \\
        \midrule
        Head Assignment Agreement (\%)
        & 100.00
        & 98.25
        \\
        \bottomrule
    \end{tabular}}
\end{table}
\begin{table}[ht]
    \centering
    \caption{Robustness analysis on the $\sigma(I^{(t,l,:)})$ coefficient. We report LLaVA-1.5-7B performance on POPE.}
    \label{tab:sigma_tl_ablation}
    \scriptsize
    \resizebox{0.5\textwidth}{!}{
    \setlength{\tabcolsep}{4pt}  
    \setlength{\extrarowheight}{0pt}
    \renewcommand{\arraystretch}{1.05}
    \begin{tabular}{lcccc}
        \toprule
        $\sigma(I^{(t,l,:)})$ coefficient
        & \textbf{1}
        & \textbf{2}
        & \textbf{3}
        & $\mathbf{\infty}$ \\
        \midrule
        F1 score $\uparrow$
        & 86.97
        & 87.21
        & \textbf{87.84}
        & 87.89 \\
        \bottomrule
    \end{tabular}}
\end{table}
\textbf{(5) Why $\mathbf{3\sigma(I^{(t,l,:)})}$ criterion.} 
The Gaussian assumption in Section \ref{causal_intervention} is not imposed on the intervention noise itself, but on the distribution of the estimated head contribution scores, which is required for applying the $3\sigma$ criterion to identify causally redundant heads.

To verify this assumption, we conduct an additional statistical analysis. For each token, we computed the contribution score of every attention head and performed a Lilliefors normality test \citep{lilliefors}, which tests the null hypothesis that a sample is drawn from a normal distribution when the mean and variance are unknown. In almost all cases, the obtained $p$-values are greater than 0.05, indicating that we cannot reject the normality hypothesis for the contribution scores.

We also applied a Yeo-Johnson transformation \citep{yeo_johnson} to further Gaussianize the score distribution and repeated the redundant-head identification. As shown in Table~\ref{tab:yeo_johnson}, the head assignment agreement before and after the transformation reaches 98.25\%, indicating that the identified redundant heads are almost unchanged. In Table \ref{tab:sigma_tl_ablation}, the F1 score is also stable under $1\sigma$, $2\sigma$, and $3\sigma$ settings, with the best achieved around the standard $3\sigma$ causal redundancy criterion.
This demonstrates that the causally redundant head identification is robust and that the $3\sigma$ criterion is well justified in practice.

From a statistical perspective, each head contribution is an aggregate effect of numerous independent factors (e.g., token context, modality interaction). Such aggregated quantities are often empirically well approximated by a Gaussian distribution, which is also consistent with the normality tests.
\subsection{Head Distribution Dynamics in an Update Interval}
\label{head_dist_interval}
\begin{figure}[ht]
    \centering
    \includegraphics[width=1.0\linewidth]{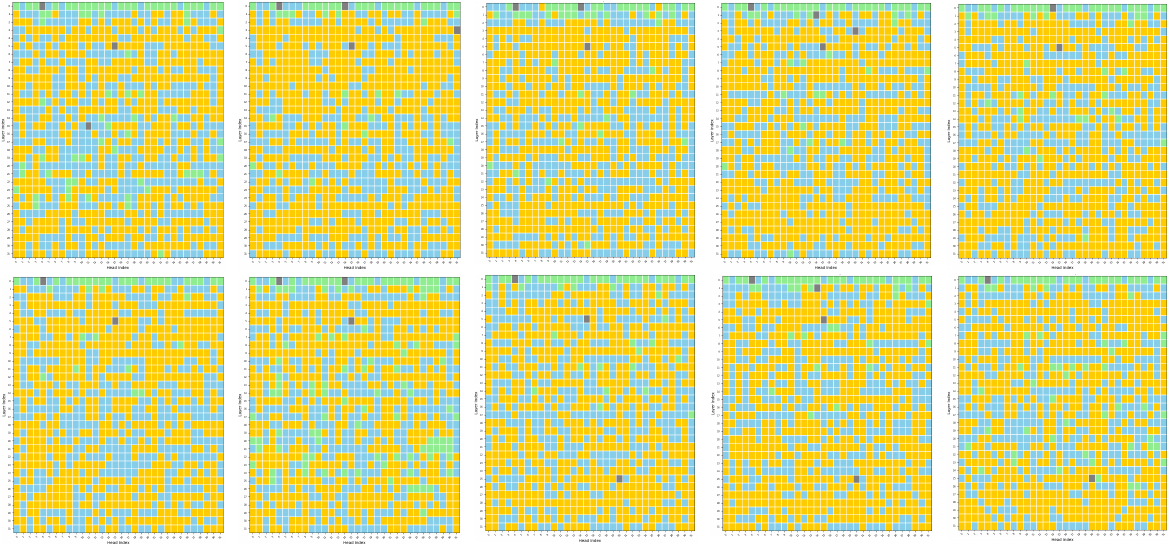}
    \caption{Head distribution in an update interval of 10 steps. We observe that the macroscopic distribution of head attributes shows negligible variation, justifying our periodic update strategy.}
    \label{fig:head_dist_interval}
\end{figure}
\subsection{Empirical Analysis on the Equilibrium Factor and Update Interval}
\label{para_selection}
At present, we cannot provide a closed-form theoretical rule for determining the optimal values of these hyperparameters. However, our empirical analysis reveals that they exhibit stable operating ranges across different MLLM families and tasks, making parameter selection easier in practice.

\begin{table}[ht]
    \centering
    \caption{LLaVA-Bench score and $C_S$ (CHAIR) under different equilibrium factors.}
    \label{tab:equilibrium_factor}
    \scriptsize
    \resizebox{1.0\textwidth}{!}{
    \setlength{\tabcolsep}{2pt}  
    \setlength{\extrarowheight}{0pt}
    \renewcommand{\arraystretch}{1.05}
    \begin{tabular}{ccccc}
        \toprule
        \textbf{Equilibrium Factor}
        & \textbf{LLaVA-Bench (LLaVA-1.5-7B) $\uparrow$}
        & \textbf{LLaVA-Bench (Qwen2.5-VL-7B) $\uparrow$}
        & \textbf{$C_S$ (LLaVA-1.5-7B) $\downarrow$}
        & \textbf{$C_S$ (Qwen2.5-VL-7B) $\downarrow$} \\
        \midrule
        0.3 & 72.1 & 75.6 & 39.0 & 23.8 \\
        0.4 & 73.6 & 78.3 & 37.1 & \textbf{23.3} \\
        0.5 & 74.6 & \textbf{78.5} & \textbf{36.9} & 23.5 \\
        0.6 & \textbf{75.2} & 78.2 & 37.4 & 24.1 \\
        0.7 & 74.8 & 77.9 & 38.1 & 24.9 \\
        \bottomrule
    \end{tabular}}
\end{table}
\begin{table}[ht]
    \centering
    \caption{LLaVA-Bench score and POPE-F1 under different update intervals.}
    \label{tab:update_interval}
    \scriptsize
    \resizebox{1.0\textwidth}{!}{
    \setlength{\tabcolsep}{2pt}  
    \setlength{\extrarowheight}{0pt}
    \renewcommand{\arraystretch}{1.05}
    \begin{tabular}{ccccc}
        \toprule
        \textbf{Update Interval}
        & \textbf{LLaVA-Bench (LLaVA-1.5-7B) $\uparrow$}
        & \textbf{LLaVA-Bench (Qwen2.5-VL-7B) $\uparrow$}
        & \textbf{POPE-F1 (LLaVA-1.5-7B) $\uparrow$}
        & \textbf{POPE-F1 (Qwen2.5-VL-7B) $\uparrow$} \\
        \midrule
        2  & 75.5 & 78.7 & 87.9 & 88.8 \\
        5  & 75.3 & 78.7 & 87.8 & 88.6 \\
        10 & \textbf{75.2} & \textbf{78.5} & \textbf{87.8} & \textbf{88.4} \\
        15 & 75.0 & 78.2 & 87.5 & 88.3 \\
        20 & 74.7 & 77.8 & 87.1 & 88.0 \\
        \bottomrule
    \end{tabular}}
\end{table}
\paragraph{Equilibrium factor.}
As shown in Table~\ref{tab:equilibrium_factor} together with Figure \ref{fig:ablation}(b), the optimal equilibrium factor consistently falls within \textbf{0.4-0.6} across both the LLaVA and Qwen model families. Although the exact optimum varies slightly across datasets, the performance remains stable within this interval.

Moreover, we observe an interesting empirical trend: models with stronger visual understanding (e.g., the Qwen series) generally achieve their best performance with a relatively smaller equilibrium factor, suggesting that such models require less additional visual calibration because their visual representations are already more reliable. In contrast, models with relatively weaker visual grounding (e.g., LLaVA) benefit from slightly larger values of $\alpha$. Therefore, rather than requiring exhaustive tuning, we recommend selecting $\alpha$ within 0.4-0.6, followed by only minor adjustments according to the visual capability of the target model.
\paragraph{Update interval.}
Similarly, Figure \ref{fig:ablation}(a) together with Table~\ref{tab:update_interval} shows that the performance is highly similar when the update interval is chosen between \textbf{5 and 15} decoding steps, while we use 10 as the default setting throughout the paper. Importantly, unlike the equilibrium factor, we find that the update interval transfers well across different tasks and datasets, indicating that it is considerably less sensitive to downstream applications.

Finally, we would like to note that these parameters are analogous to commonly used inference-time hyperparameters such as temperature or top-$p$, which are intentionally exposed to users to flexibly control the decoding behavior. In our case, the equilibrium factor provides an interpretable control over the balance between visual evidence and language information. As theoretically supported by Theorem \ref{theorem_2}, increasing $\alpha$ continuously shifts the output representation toward the visual direction, allowing users to explicitly control how strongly the model relies on visual evidence.

\begin{table}[ht]
    \centering
    \caption{Intern-VL performance comparison with our and other methods.}
    \label{tab:internvl_comparison}
    \scriptsize
    \resizebox{0.9\textwidth}{!}{
    \setlength{\tabcolsep}{4pt}  
    \setlength{\extrarowheight}{0pt}
    \renewcommand{\arraystretch}{1.05}
    \begin{tabular}{lcccccc}
        \toprule
        \textbf{Method}
        & \textbf{LLaVA-Bench$\uparrow$}
        & \textbf{CHAIR$_S\downarrow$}
        & \textbf{CHAIR$_I\downarrow$}
        & \textbf{POPE-R$\uparrow$}
        & \textbf{POPE-F1$\uparrow$}
        & \textbf{POPE-A$\uparrow$} \\
        \midrule
        InternVL-7B
        & 51.6 & 46.6 & 12.4 & 80.0 & 85.3 & 86.2 \\
        +LocoRE~\citep{hallu_saliency}
        & 52.8 & 40.2 & 10.5 & 85.8 & 87.2 & 87.3 \\
        +HEAL
        & \textbf{53.4}
        & \textbf{39.2}
        & \textbf{9.6}
        & \textbf{86.5}
        & \textbf{87.8}
        & \textbf{87.9} \\
        \bottomrule
    \end{tabular}}
\end{table}
To further verify the transferability of these practical guidelines, we follow the above strategy for the other models in our experiments. For example, Intern-VL adopts the same parameter setting as the LLaVA family due to their similar visual capability. Although we do not separately tune these hyperparameters for Intern-VL, HEAL still consistently improves all evaluation metrics (Table~\ref{tab:internvl_comparison}), suggesting that the proposed parameter ranges generalize well to other MLLMs.
\subsection{Performance on Recent Benchmarks and MLLMs}
\begin{table}[ht]
    \centering
    \caption{HEAL performance on recent benchmarks and MLLMs.}
    \label{tab:recent_benchmarks}
    \footnotesize
    \resizebox{0.7\textwidth}{!}{
    \setlength{\tabcolsep}{2pt}
    \setlength{\extrarowheight}{0pt}
    \renewcommand{\arraystretch}{1.1}
    \begin{tabular}{lccccccc}
        \toprule
        \multirow{2}{*}{\textbf{Method}}
        & \multicolumn{2}{c}{\textbf{MMHal-Bench}}
        & \multicolumn{5}{c}{\textbf{BLINK-Twice}} \\
        \cmidrule(lr){2-3}
        \cmidrule(lr){4-8}
        &
        \textbf{Halluc. Rate}$\downarrow$
        & \textbf{Score}$\uparrow$
        & \textbf{No-Acc}
        & \textbf{Yes-Acc}
        & \textbf{Q-Acc}
        & \textbf{I-Acc}
        & \textbf{G-Acc} \\
        \midrule
        Qwen3-VL-8B
        & 17.5 & 4.82 & 0.476 & 0.653 & 0.493 & 0.336 & 0.140 \\
        \rowcolor{blue!04}
        \textbf{+HEAL}
        & 16.6 & 4.91 & 0.557 & 0.714 & 0.536 & 0.353 & 0.213 \\
        \hline
        InternVL3.5-8B
        & 19.4 & 4.53 & 0.355 & 0.588 & 0.475 & 0.264 & 0.109 \\
        \rowcolor{blue!04}
        \textbf{+HEAL}
        & 18.3 & 4.71 & 0.463 & 0.624 & 0.583 & 0.334 & 0.158 \\
        \bottomrule
    \end{tabular}}
\end{table}
We conducted additional experiments on Qwen3-VL-8B \citep{qwen3vl} and InternVL3.5-8B \citep{internvl3.5}, two recent open-source MLLMs, and evaluated them on MMHal-Bench \citep{mmhal_bench} (ACL 2024) and BLINK-Twice \citep{blink_twice} (NeurIPS 2026), which provide more challenging evaluations of multimodal hallucination and visual reasoning than earlier benchmarks.

As shown in Table \ref{tab:recent_benchmarks}, HEAL consistently improves performance across both models and both benchmarks. On Qwen3-VL-8B, HEAL reduces the hallucination rate from 17.5 to 16.6 while improving the MMHal score from 4.82 to 4.91. On InternVL3.5-8B, the hallucination rate decreases from 19.4 to 18.3, and the MMHal score improves from 4.53 to 4.71. On BLINK-Twice, HEAL consistently improves all evaluation metrics, including No-Acc, Yes-Acc, Q-Acc, I-Acc, and G-Acc, showing that the proposed dynamic information calibration generalizes well to recent MLLMs and challenging visual reasoning scenarios.
\section{Further Discussions on HEAL}
\textbf{(1) Do not resolve hallucinations that stem from early visual encoding failures or a fundamental lack of visual evidence in the input.}

HEAL is an inference-time calibration framework that operates on the internal attention representations during decoding. Its objective is to calibrate the contributions of visual and language information when both sources of information are already available in the model. Consequently, if the visual encoder fails to extract reliable visual features, or if the input itself lacks sufficient visual evidence, there is no additional visual information that can be recovered through attention calibration alone.

\textbf{(2) HEAL applies a single, global equilibrium factor to all synergy heads. Do all synergy heads truly behave uniformly?}

In fact, our empirical observations show that almost every synergy head exhibits its own attention pattern and modality preference.
The key motivation is that HEAL is designed to correct a global decoding imbalance rather than optimize each head independently. For each token, the final prediction is jointly determined by the aggregation of all attention heads. Consequently, the objective of HEAL is to regulate the \textit{collective} distribution of visual and language information at the token level, making a global equilibrium factor a natural design choice.

Although a few synergy heads may exhibit highly specialized behaviors, applying an identical equilibrium factor does not require every head to respond identically. Instead, it provides a consistent global bias that shifts the overall attention representation toward a better visual-language balance. While the contribution of a small number of individual heads may become slightly suboptimal, this effect is compensated by the remaining heads during multi-head aggregation. Since the decoder prediction depends on the \textit{collective} output of all heads rather than any individual head, the overall generation quality is preserved, and hallucinations are effectively reduced.

Our empirical results support this design. As shown in the hyperparameter analysis (Figure \ref{fig:ablation}(b), Tables \ref{tab:equilibrium_factor} and \ref{tab:internvl_comparison}), a single equilibrium factor consistently improves hallucination-related metrics across different model families, including LLaVA, Qwen, and Intern-VL. These results suggest that although synergy heads are heterogeneous at the individual level, their \textit{collective} imbalance can be effectively corrected through a shared global calibration, which is sufficient to achieve consistent performance improvements across diverse MLLMs.

\textbf{(3) The method updates head types every 10 steps based on temporal locality. However, complex queries often shift abruptly between visual description and linguistic reasoning. Does this fixed interval delay adaptation? Could it cause localized hallucinations during these sharp transitions?}

The fixed update interval inevitably introduces a certain delay in tracking the dynamic changes of attention head attributes. This is an intentional design choice that balances adaptation accuracy and inference efficiency. Updating the head taxonomy at every decoding step would provide the most up-to-date estimation, but it would also incur a substantially higher computational cost, making the method much less practical for inference-time deployment.
Importantly, our empirical observations suggest that this delay has only a limited impact on the final decoding behavior. As discussed above, although the types of synergy heads evolve dynamically during token generation, the visual heads remain highly sparse and remarkably stable, while most transitions occur between language-oriented and synergy heads. Similar observations have also been reported in recent studies on attention head dynamics in MLLMs \citep{sparsemm,vhr}. Consequently, a moderate update interval is sufficient to capture the dominant head organization throughout decoding.

Furthermore, our hyperparameter analysis (Figure \ref{fig:ablation}(a) and Table \ref{tab:update_interval}) shows that update intervals between 5 and 15 decoding steps consistently achieve very similar performance, indicating that HEAL is not particularly sensitive to moderate delays in taxonomy updates. We therefore use 10 decoding steps as the default setting, which provides a favorable balance between computational overhead and hallucination mitigation.
Regarding the concern about rapidly switching reasoning patterns (e.g., alternating between visual grounding and language reasoning), we acknowledge that existing benchmarks do not explicitly isolate this phenomenon, making it difficult to quantify its impact directly. Nevertheless, we further evaluate HEAL on more challenging benchmarks, like MMHal-Bench \citep{mmhal_bench}, involving complex multimodal reasoning and localization, where frequent transitions are more likely to occur. The consistent improvements in Table \ref{tab:recent_benchmarks} suggest that the delayed updates introduce only limited degradation in practice.

\textbf{(4) About the two redundancy definitions in the paper.}

In Section \ref{causal_intervention}, redundancy is defined from the perspective of causal contribution. Specifically, after replacing one head output with its counterfactual counterpart, we measure the change in the final multi-head attention representation. If a head has negligible influence on the final representation, it is considered a \textit{causally redundant head}. This filtering step is necessary because some heads may have very limited contribution after the output projection layer, and retaining these heads can introduce noise into the subsequent head taxonomy.

In contrast, Section \ref{did} focuses on information composition within each remaining head. Here, an \textit{information redundant head} refers to a head whose information content does not change significantly after introducing visual or language tokens. In other words, these heads contain negligible modality-specific information and therefore cannot provide meaningful visual-language decomposition.

Both filtering stages are necessary. As shown in Table \ref{tab:replacement_robust}, removing the causal redundancy filtering affects the subsequent head taxonomy and final performance. Meanwhile, without the information redundancy filtering, the modality decomposition becomes unreliable because heads with almost zero information variation cannot provide meaningful visual-language information ratios.

\textbf{(5) The difference between the MoE and dense models in the design of the equilibrium factor. For the MoE model with different activated experts in use, how to compute the value of the equilibrium factor $\alpha$?}

HEAL operates entirely at the Multi-Head Attention (MHA) level rather than the FFN/MLP level. Specifically, our causal intervention, Difference-in-Differences analysis, head taxonomy, and dynamic information calibration are all performed on the attention outputs and value vectors (see Figure \ref{fig:heal} and Theorem \ref{theorem_1}). In contrast, sparse MoE architectures replace the feed-forward (MLP/FFN) sublayer with routed experts, while the attention module remains unchanged in mainstream MoE MLLMs. Therefore, the computational procedure of HEAL is identical for dense and MoE models, and is independent of which experts are activated during inference.

Regarding the equilibrium factor, $\alpha$ is a model-level hyperparameter rather than a quantity computed from the activated experts. Its role is to control the target balance between visual and language information during calibration. Once $\alpha$ is determined, the calibration factors are computed solely from the visual/language information proportions estimated for each attention head, which are independent of the downstream MoE routing. Consequently, the computation of both the equilibrium factor and the calibration factors is the same for dense and MoE models. In practice, $\alpha$ is selected using the same procedure as in dense models, since models within the same family (e.g., Qwen-VL dense and sparse variants) exhibit similar multimodal information characteristics.

\textbf{(6) Does HEAL employ the counterfactual mechanism for all the attention layers?}

HEAL performs the counterfactual analysis for all attention layers. The reason is that our central finding is global rather than layer-specific: hallucinations are causally associated with the disequilibrium between visual and language information within synergy heads distributed throughout the entire network, rather than being dominated by a few specific layers. Therefore, identifying synergy heads requires performing the proposed DiD analysis across all attention layers. Based on the identified head taxonomy, HEAL dynamically calibrates only the synergy heads, which account for a small portion of all attention heads, while the remaining heads are left unchanged.

Although the head taxonomy is obtained from all attention layers, the computational overhead during inference is greatly reduced by the optimizations described in Section \ref{efficient_imple}. First, head attributes are updated periodically rather than at every decoding step, leveraging the temporal locality of head behaviors. Second, the counterfactual analysis is implemented in a batched and parallelized manner. During intermediate decoding steps without taxonomy updates, HEAL only computes the calibration factors for the identified synergy heads instead of repeating the full analysis.

\textbf{(7) Geometric interpretation of HEAL.}

It is important to distinguish the geometric representation space from the information decomposition space in HEAL. In the representation space, the output of a multimodal attention head can be locally viewed as a combination of visual and language components together with a residual component,
\[
h = r + v + l,
\]
where $v$ and $l$ denote the visual and language representation components, and $r$ collects residual components that are not explicitly attributed to either modality.

In contrast, the quantities $I_{\mathrm{vis}}$, $I_{\mathrm{lang}}$, and $I_{\mathrm{syn}}$ computed in Section \ref{did} characterize the information structure induced by counterfactual interventions. In particular,
$
I_{\mathrm{syn}}
=
I_{\mathrm{total}}
-
I_{\mathrm{vis}}
-
I_{\mathrm{lang}}
$
represents the interaction between visual and language information that cannot be accounted for by their individual contributions. Thus, $I_{\mathrm{syn}}$ should not be interpreted as an additional Euclidean vector added independently to $v$ and $l$. Rather, it quantifies a cross-modal interaction in the information space. This distinction explains the geometric illustration in Figure \ref{fig:heal}. The figure depicts the representation space effect of calibration rather than a vector decomposition of all information atoms.

\end{document}